\documentclass[11pt,oneside]{article}

\pdfoutput=1 % For arXiv; to ensure it is compiled with PDFLaTeX.

\usepackage{style_mjh_arXiv}
\usepackage[numbers]{natbib}
\usepackage{amsmath}                                          
\usepackage{amsthm}                                           
\theoremstyle{definition} \newtheorem{defn}{Definition}       
\theoremstyle{plain} \newtheorem{prop}[defn]{Proposition}     
\theoremstyle{plain} \newtheorem{thm}[defn]{Theorem}          
\theoremstyle{plain} \newtheorem{lem}[defn]{Lemma}            
\theoremstyle{plain}         
\theoremstyle{remark} \newtheorem{rmk}[defn]{Remark}          
\theoremstyle{remark}

\usepackage{enumitem}
\makeatletter
\def\namedlabel#1#2{\begingroup
    #2%
    \def\@currentlabel{#2}%
    \phantomsection\label{#1}\endgroup
}
\makeatother

\usepackage[pdfauthor={Matthew James Holland},%
pdftitle={Holland (2019): PAC-Bayes under potentially heavy tails},%
colorlinks=true,% 
linkcolor=blue,%
citecolor=blue,%
pdftex]{hyperref} %

\begin{document}

\title{\textbf{PAC-Bayes under potentially heavy tails}}
\author{
  Matthew J.~Holland\thanks{Please direct correspondence to \texttt{matthew-h@ar.sanken.osaka-u.ac.jp}.}\\
  Osaka University
}
\date{} % empty date.

\maketitle

\begin{abstract}
We derive PAC-Bayesian learning guarantees for heavy-tailed losses, and obtain a novel optimal Gibbs posterior which enjoys finite-sample excess risk bounds at logarithmic confidence. Our core technique itself makes use of PAC-Bayesian inequalities in order to derive a robust risk estimator, which by design is easy to compute. In particular, only assuming that the first three moments of the loss distribution are bounded, the learning algorithm derived from this estimator achieves nearly sub-Gaussian statistical error, up to the quality of the prior.
\end{abstract}

\section{Introduction}\label{sec:intro}

More than two decades ago, the origins of PAC-Bayesian learning theory were developed with the goal of strengthening traditional PAC learning guarantees\footnote{PAC: Probably approximately correct \citep{valiant1984a}.} by explicitly accounting for prior knowledge \citep{shawetaylor1996a,mcallester1999a,catoni2004SLT}. Subsequent work developed finite-sample risk bounds for ``Bayesian'' learning algorithms which specify a distribution over the model \citep{mcallester2003a}. These bounds are controlled using the empirical risk and the relative entropy between ``prior'' and ``posterior'' distributions, and hold uniformly over the choice of the latter, meaning that the guarantees hold for data-dependent posteriors, hence the naming. Furthermore, choosing the posterior to minimize PAC-Bayesian risk bounds leads to practical learning algorithms which have seen numerous successful applications \citep{alquier2016a}.

Following this framework, a tremendous amount of work has been done to refine, extend, and apply the PAC-Bayesian framework to new learning problems. Tight risk bounds for bounded losses are due to \citet{seeger2002a} and \citet{maurer2004a}, with the former work applying them to Gaussian processes. Bounds constructed using the loss variance in a Bernstein-type inequality were given by \citet{seldin2012a}, with a data-dependent extension derived by \citet{tolstikhin2013a}. As stated by \citet{mcallester2013a}, virtually all the bounds derived in the original PAC-Bayesian theory ``only apply to bounded loss functions.'' This technical barrier was solved by \citet{alquier2016a}, who introduce an additional error term depending on the concentration of the empirical risk about the true risk. This technique was subsequently applied to the log-likelihood loss in the context of Bayesian linear regression by \citet{germain2016a}, and further systematized by \citet{begin2016a}. While this approach lets us deal with unbounded losses, naturally the statistical error guarantees are only as good as the confidence intervals available for the empirical mean deviations. In particular, strong assumptions on all of the moments of the loss are essentially unavoidable using the traditional tools espoused by \citet{begin2016a}, which means the ``heavy-tailed'' regime, where all we assume is that a few higher-order moments are finite (say finite variance and/or finite kurtosis), cannot be handled. A new technique for deriving PAC-Bayesian bounds even under heavy-tailed losses is introduced by \citet{alquier2018a}; their lucid procedure provides error rates even under heavy tails, but as the authors recognize, the rates are highly sub-optimal due to direct dependence on the empirical risk, leading in turn to sub-optimal algorithms derived from these bounds.\footnote{See work by \citet{catoni2012a,devroye2016a} and the references within for background on the fundamental limitations of the empirical mean for real-valued random variables.}

In this work, while keeping many core ideas of \citet{begin2016a} intact, using a novel approach we obtain exponential tail bounds on the excess risk using PAC-Bayesian bounds that hold even under heavy-tailed losses. Our key technique is to replace the empirical risk with a new mean estimator inspired by the dimension-free estimators of \citet{catoni2017a}, designed to be computationally convenient. We review some key theory in section \ref{sec:bg} before introducing the new estimator in section \ref{sec:estimator}. In section \ref{sec:bounds} we apply this estimator to the PAC-Bayes setting, deriving a new robust optimal Gibbs posterior. Most detailed proofs are relegated to section \ref{sec:tech_proofs} in the appendix.

\section{PAC-Bayesian theory based on the empirical mean}\label{sec:bg}

Let us begin by briefly reviewing the best available PAC-Bayesian learning guarantees under general losses. Denote by $\zz_{1},\ldots,\zz_{n} \in \ZZ$ a sequence of independent observations distributed according to common distribution $\ddist$. Denote by $\HH$ a model/hypothesis class, from which the learner selects a candidate based on the $n$-sized sample. The quality of this choice can be measured in a pointwise fashion using a loss function $l: \HH \times \ZZ \to \RR$, assumed to be $l \geq 0$. The learning task is to achieve a small risk, defined by $\risk(h) \defeq \exx_{\ddist}l(h;\zz)$. Since the underlying distribution is inherently unknown, the canonical proxy is
\begin{align*}
\emrisk(h) \defeq \frac{1}{n} \sum_{i=1}^{n} l(h;\zz_{i}), \quad h \in \HH.
\end{align*}
Let $\prior$ and $\post$ respectively denote ``prior'' and ``posterior'' distributions on the model $\HH$. The so-called Gibbs risk induced by $\post$, as well as its empirical counterpart are given by
\begin{align*}
\grisk_{\post} \defeq \exx_{\post} \risk = \int_{\HH} \risk(h) \, d\post(h), \quad \emgrisk_{\post} \defeq \exx_{\post} \emrisk = \frac{1}{n} \sum_{i=1}^{n} \int_{\HH} l(h;\zz_{i}) \, d\post(h).
\end{align*}
When our losses are almost surely bounded, lucid guarantees are available.
\begin{thm}[PAC-Bayes under bounded losses \citep{mcallester2003a,begin2016a}]\label{thm:PB_bounded_loss}
Assume $0 \leq l \leq 1$, and fix any arbitrary prior $\prior$ on $\HH$. For any confidence level $\delta \in (0,1)$, we have with probability no less than $1-\delta$ over the draw of the sample that
\begin{align*}
\grisk_{\post} \leq \emgrisk_{\post} + \sqrt{\frac{\KL(\post;\prior) + \log(2\sqrt{n}\delta^{-1})}{2n}}
\end{align*}
uniformly in the choice of $\post$.
\end{thm}
\noindent Since the ``good event'' where the inequality in Theorem \ref{thm:PB_bounded_loss} holds is valid for any choice of $\post$, the result holds even when $\post$ depends on the sample, which justifies calling it a posterior distribution. Optimizing this upper bound with respect to $\post$ leads to the so-called optimal Gibbs posterior, which takes a form which is readily characterized (cf.~Remark \ref{rmk:compare_traditional_gibbs}).

The above results fall apart when the loss is unbounded, and meaningful extensions become challenging when exponential moment bounds are not available. As highlighted in section \ref{sec:intro} above, over the years, the analytical machinery has evolved to provide general-purpose PAC-Bayesian bounds even under heavy-tailed data. The following theorem of \citet{alquier2018a} extends the strategy of \citet{begin2016a} to obtain bounds under the weakest conditions we know of.
\begin{thm}[PAC-Bayes under heavy-tailed losses \citep{alquier2018a}]
Take any $p > 1$ and set $q = p / (p-1)$. For any confidence level $\delta \in (0,1)$, we have with probability no less than $1-\delta$ over the draw of the sample that
\begin{align*}
\grisk_{\post} \leq \emgrisk_{\post} + \left( \frac{\exx_{\prior}|\emrisk-\risk|^{q}}{\delta} \right)^{\frac{1}{q}} \left( \int_{\HH} \left(\frac{d\post}{d\prior}\right)^{p} d\prior \right)^{\frac{1}{p}}
\end{align*}
uniformly in the choice of $\post$.
\end{thm}
\noindent For concreteness, consider the case of $p=2$, where $q=2/(2-1)=2$, and assume that the variance of the loss is $\vaa_{\ddist} l(h;\zz)$ is $\prior$-finite, namely that
\begin{align*}
V_{\prior} \defeq \int_{\HH} \vaa_{\ddist} l(h;\zz) \, d\prior(h) < \infty.
\end{align*}
From Proposition 4 of \citet{alquier2018a}, we have $\exx_{\prior}|\emrisk-\risk|^{2} \leq V_{\prior}/n$. It follows that on the high-probability event, we have
\begin{align*}
\grisk_{\post} \leq \emgrisk_{\post} + \sqrt{\frac{V_{\prior}}{n\,\delta} \left(\int_{\HH} \left(\frac{d\post}{d\prior}\right)^{2} d\prior\right) }
\end{align*}
While the $\sqrt{n}$ rate and dependence on a divergence between $\prior$ and $\post$ are similar, note that the dependence on the confidence level $\delta \in (0,1)$ is polynomial; compare this with the logarithmic dependence available in Theorem \ref{thm:PB_bounded_loss} above when the losses were bounded.

For comparison, our main result of section \ref{sec:bounds} is a uniform bound on the Gibbs risk: with probability no less than $1-\delta$, we have
\begin{align*}
\grisk_{\post} \leq \emgrisk_{\post,\trunc} + \frac{1}{\sqrt{n}}\left( \KL(\post;\prior) + \frac{\log(8\pi \bdmnt \delta^{-2})}{2} + \bdmnt + \prior_{n}^{\ast}(\HH) - 1 \right) + O\left(\frac{1}{n}\right)
\end{align*}
where $\emgrisk_{\post,\trunc}$ is an estimator of $\grisk_{\post}$ defined in section \ref{sec:estimator}, $\prior_{n}^{\ast}(\HH)$ is a term depending on the quality of prior $\prior$, and the key constants are bounds such that for all $h \in \HH$ we have $\bdmnt \geq \exx_{\ddist}l(h;\zz)^{2}$. As long as the first three moments are finite, this guarantee holds, and thus both sub-Gaussian and heavy-tailed losses (e.g., with infinite higher-order moments) are permitted. Given any valid $\bdmnt$, the PAC-Bayesian upper bound above can be minimized in $\post$ based on the data, and thus an optimal Gibbs posterior can also be computed in practice. In section \ref{sec:bounds}, we characterize this ``robust posterior.''

\section{A new estimator using smoothed Bernoulli noise}\label{sec:estimator}

\paragraph{Notation}

In this section, we are dealing with the specific problem of robust mean estimation, thus we specialize our notation here slightly. Data observations will be $x_{1},\ldots,x_{n} \in \RR$, assumed to be independent copies of $x \sim \ddist$. Denote the index set $[k] \defeq \{1,2,\ldots,k\}$. Write $\proball(\Omega,\AAcal)$ for the set of all probability measures defined on the measurable space $(\Omega,\AAcal)$. Write $\KL(P,Q)$ for the relative entropy between measures $P$ and $Q$ (also known as the KL divergence; definition in appendix). We shall typically suppress $\AAcal$ and even $\Omega$ in the notation when it is clear from the context. Let $\trunc$ be a bounded, non-decreasing function such that for some $b>0$ and all $u \in \RR$,
\begin{align}\label{eqn:trunc_keyprop}
-\log\left(1 - u + u^{2}/b\right) \leq \trunc(u) \leq \log\left(1 + u + u^{2}/b\right).
\end{align}
As a concrete and analytically useful example, we shall use the piecewise polynomial function of \citet{catoni2017a}, defined by
\begin{align}\label{eqn:trunc_CG17_defn}
\trunc(u) \defeq
\begin{cases}
u - u^{3}/6, & -\sqrt{2} \leq u \leq \sqrt{2}\\
2\sqrt{2}/3, & u > \sqrt{2}\\
-2\sqrt{2}/3, & u < -\sqrt{2}
\end{cases}
\end{align}
which satisfies (\ref{eqn:trunc_keyprop}) with $b=2$, and is pictured in Figure \ref{fig:trunc_catgiu} with the two key bounds.\footnote{Slightly looser bounds hold for an analogous procedure using a Huber-type influence function with $b=1$.}

\begin{figure}[t]
\centering
\includegraphics[width=0.5\textwidth]{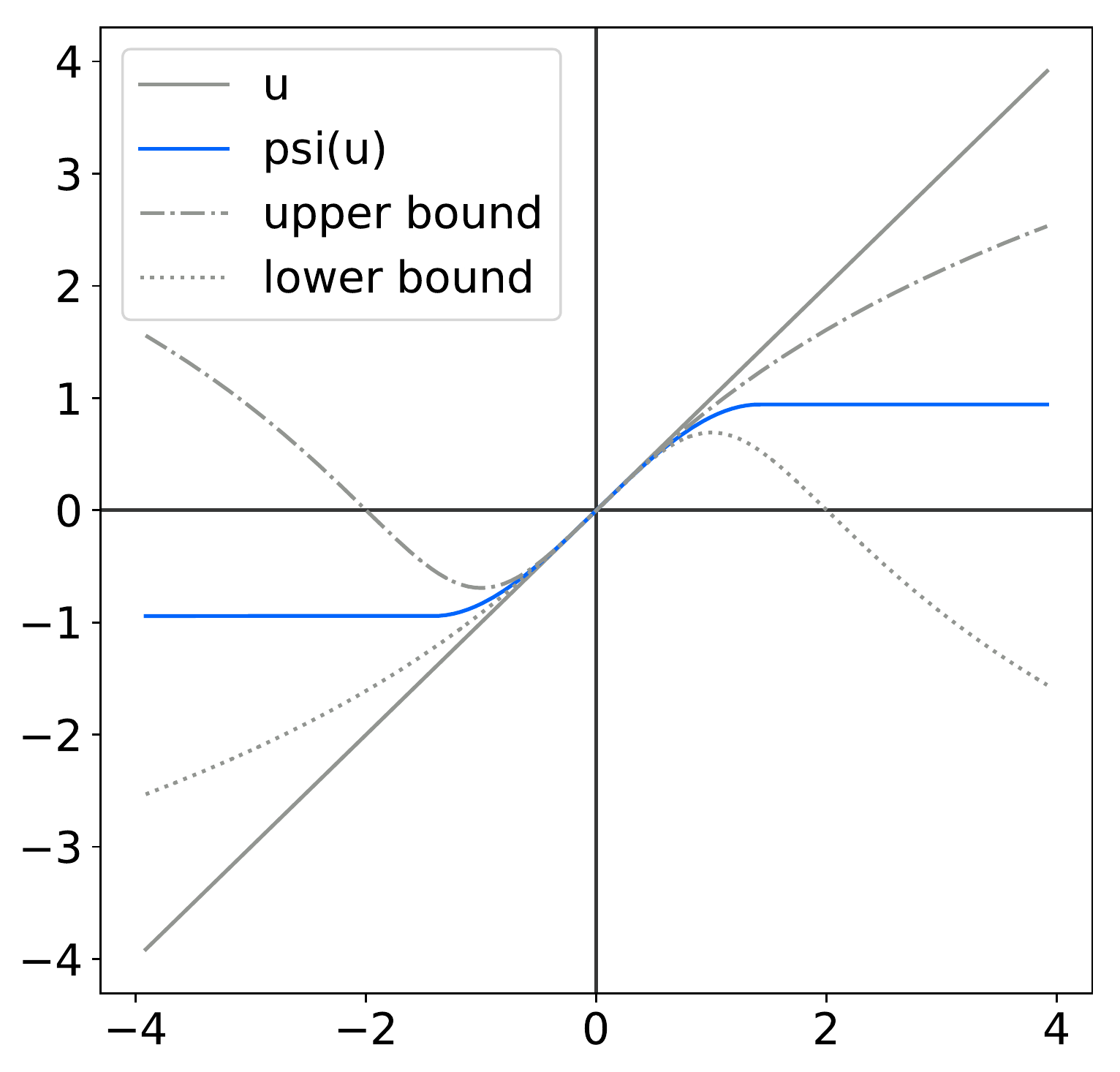}
\caption{Graph of the Catoni function $\trunc(u)$ over $\pm \sqrt{2} \pm 2.5$.}
\label{fig:trunc_catgiu}
\end{figure}

\paragraph{Estimator definition}

We consider a straightforward procedure, in which the data are subject to a soft truncation after re-scaling, defined by
\begin{align}\label{eqn:defn_est_bernoulli}
\xhat \defeq \frac{s}{n} \sum_{i=1}^{n} \trunc\left(\frac{x_{i}}{s}\right)
\end{align}
where $s>0$ is a re-scaling parameter. Depending on the setting of $s$, this function can very closely approximate the sample mean, and indeed modifying this scaling parameter controls the bias of this estimator in a direct way, which can be quantified as follows. As the scale grows, note that
\begin{align*}
s \trunc\left(\frac{x}{s}\right) = x - \frac{x^{3}}{6s^{2}} \to x, \quad \text{ as } s \to \infty
\end{align*}
which implies that taking expectation with respect to the sample and $s \to \infty$, in the limit this estimator is unbiased, with
\begin{align*}
\exx \left( \frac{s}{n} \sum_{i=1}^{n} \trunc\left(\frac{x_{i}}{s}\right) \right) = \exx_{\ddist} x - \frac{\exx_{\ddist} x^{3}}{6s^{2}} \to \exx_{\ddist} x.
\end{align*}
On the other hand, taking $s$ closer to zero implies that more observations will be truncated. Taking $s$ small enough,\footnote{More precisely, taking $s \leq \min\{|x_{i}|: i \in [n]\}/\sqrt{2}$.} we have 
\begin{align*}
\frac{s}{n} \sum_{i=1}^{n} \trunc\left(\frac{x_{i}}{s}\right) = \frac{2\sqrt{2}s}{3n}\left(|\II_{+}|-|\II_{-}|\right),
\end{align*}
which converges to zero as $s \to 0$. Here the positive/negative indices are $\II_{+} \defeq \{i \in [n]: x_{i} > 0\}$ and $\II_{-} \defeq \{i \in [n]: x_{i} < 0\}$. Thus taking $s$ too small means that only the signs of the observations matter, and the absolute value of the estimator tends to become too small.

\paragraph{High-probability deviation bounds for $\xhat$}

We are interested in high-probability bounds on the deviations $|\xhat-\exx_{\ddist} x|$ under the weakest possible assumptions on the underlying data distribution. To obtain such guarantees in a straightforward manner, we make the simple observation that the estimator $\xhat$ defined in (\ref{eqn:defn_est_bernoulli}) can be related to an estimator with smoothed noise as follows. Let $\epsilon_{1},\ldots,\epsilon_{n}$ be an iid sample of noise $\epsilon \in \{0,1\}$ with distribution $\text{Bernoulli}(\theta)$ for some $0 < \theta < 1$. Then, taking expectation with respect to the noise sample, one has that
\begin{align}\label{eqn:bridge_identity}
\xhat = \frac{1}{\theta} \, \exx\left( \frac{s}{n} \sum_{i=1}^{n} \trunc\left(\frac{x_{i}\,\epsilon_{i}}{s}\right) \right).
\end{align}
This simple observation becomes useful to us in the context of the following technical fact.

\begin{lem}\label{lem:cat17type_KLbound}
Assume we are given some independent data $x_{1},\ldots,x_{n}$, assumed to be copies of the random variable $x \sim \ddist$. In addition, let $\epsilon_{1},\ldots,\epsilon_{n}$ similarly be independent observations of ``strategic noise,'' with distribution $\epsilon \sim \post$ that we can design. Fix an arbitrary prior distribution $\prior$, and consider $f:\RR^{2} \to \RR$, assumed to be bounded and measurable. Write $\KL(\post;\prior)$ for the Kullback-Leibler divergence between distributions $\post$ and $\prior$. It follows that with probability no less than $1-\delta$ over the random draw of the sample, we have
\begin{align*}
\exx \left(\frac{1}{n} \sum_{i=1}^{n} f(x_{i},\epsilon_{i})\right) \leq \int \log\exx_{\ddist}\exp(f(x,\epsilon)) \, d\post(\epsilon) + \frac{\KL(\post;\prior) + \log(\delta^{-1})}{n},
\end{align*}
uniform in the choice of $\post$, where expectation on the left-hand side is over the noise sample.
\end{lem}

\noindent The special case of interest here is $f(x,\epsilon) = \trunc(x\epsilon/s)$. Using (\ref{eqn:trunc_keyprop}) and Lemma \ref{lem:cat17type_KLbound}, with prior $\prior = \text{Bernoulli}(1/2)$ and posterior $\post = \text{Bernoulli}(\theta)$, it follows that on the $1-\delta$ high-probability event, uniform in the choice of $0 < \theta < 1$, we have
\begin{align}
\label{eqn:xhat_KL_bound}
\left(\frac{\theta}{s}\right)\xhat & \leq \int \left( \frac{\epsilon \exx_{\ddist}x}{s} + \frac{\epsilon^{2}\exx_{\ddist}x^{2}}{2s^{2}} \right) \, d\post(\epsilon) + \frac{\KL(\post;\prior) + \log(\delta^{-1})}{n}\\
\nonumber
& = \frac{\theta \exx_{\ddist}x}{s} + \frac{\theta \exx_{\ddist} x^{2}}{2s^{2}} + \frac{1}{n} \left(\theta \log(2\theta) + (1-\theta)\log(2(1-\theta)) + \log(\delta^{-1})\right)
\end{align}
where we have used the fact that $\exx \epsilon^{2} = \exx \epsilon = \theta$ in the Bernoulli case. Dividing both sides by $(\theta/s)$ and optimizing this as a function of $s>0$ yields a closed-form expression for $s$ depending on the second moment, the confidence $\delta$, and $\theta$. Analogous arguments yield lower bounds on the same quantity. Taking these facts together, we have the following proposition, which says that assuming only finite second moments $\exx_{\ddist} x^{2} < \infty$, the proposed estimator achieves exponential tail bounds scaling with the second non-central moment.
\begin{prop}[Concentration of deviations]\label{prop:devdb_pointwise}
Scaling with $s^{2} = n \exx_{\ddist} x^{2} / 2\log(\delta^{-1})$, the estimator defined in (\ref{eqn:defn_est_bernoulli}) satisfies
\begin{align}\label{eqn:bound_new_1d}
|\xhat-\exx_{\ddist}x| \leq \sqrt{\frac{2\exx_{\ddist}x^{2}\log(\delta^{-1})}{n}}
\end{align}
with probability at least $1-2\delta$.
\end{prop}
\begin{proof}[Proof of Proposition \ref{prop:devdb_pointwise}]
First, note that the upper bound derived from (\ref{eqn:xhat_KL_bound}) holds uniformly in the choice of $\theta$ on a $(1-\delta)$ high-probability event. Setting $\theta = 1/2$ and solving for the optimal $s>0$ setting is just calculus. It remains to obtain a corresponding lower bound on $\xhat - \exx_{\ddist}x$. To do so, consider the analogous setting of Bernoulli $\prior$ and $\post$, but this time on the domain $\{-1,0\}$, with $\post\{-1\}=\theta$ and $\prior\{-1\}=1/2$. Using (\ref{eqn:trunc_keyprop}) and Lemma \ref{lem:cat17type_KLbound} again, we have
\begin{align*}
\left(\frac{-\theta}{s}\right)\xhat & \leq \frac{-\theta \exx_{\ddist}x}{s} + \frac{\theta \exx_{\ddist} x^{2}}{2s^{2}} + \frac{1}{n} \left(\theta \log(2\theta) + (1-\theta)\log(2(1-\theta)) + \log(\delta^{-1})\right)
\end{align*}
where we note $\exx_{\post}\epsilon = -\theta$ and  $\exx_{\post}\epsilon^{2} = \exx_{\post}|\epsilon| = \theta$. This yields a high-probability lower bound in the desired form when we set $\theta = 1/2$, since an upper bound on $-\xhat + \exx_{\ddist}x$ is equivalent to a lower bound on $\xhat-\exx_{\ddist}x$. However, since we have changed the prior in this case, the high-probability event here need not be the same as that for the upper bound, and as such, we must take a union bound over these two events to obtain the desired final result.
\end{proof}

\begin{rmk}
While the above bound (\ref{eqn:bound_new_1d}) depends on the true second moment, as is clear from the proof outlined above, the result is easily extended to hold for any valid upper bound on the moment, which is what will inevitably be used in practice.
\end{rmk}

\paragraph{Centered estimates}

Note that the bound (\ref{eqn:bound_new_1d}) depends on the second moment of the underlying data; this is in contrast to M-estimators which due to a natural ``centering'' of the data typically have tail bounds depending on the variance. This results in a sensitivity to the absolute value of the location of the distribution, e.g., on a distribution with unit variance and $\exx_{\ddist} x = 0$ will tend to be much better than a distribution with $\exx_{\ddist} x = 10^4$. Fortunately, a simple centering strategy works well to alleviate this sensitivity, as follows.

Without loss of generality, assume that the first $0 < k < n$ estimates are used for constructing a shifting device, with the remaining $n-k > 0$ points left for running the usual routine on shifted data. More concretely, define
\begin{align}
\label{eqn:xm_centered_shift}
\overbar{x}_{\trunc} = \frac{\overbar{s}}{k} \sum_{i=1}^{k}\trunc\left(\frac{x_{i}}{\overbar{s}}\right), \text{ where } \overbar{s}^{2} = \frac{k \exx_{\ddist}x^{2}}{2\log(\delta^{-1})}.
\end{align}
From (\ref{eqn:bound_new_1d}) in Proposition \ref{prop:devdb_pointwise}, we have
\begin{align*}
|\overbar{x}_{\trunc} - \exx_{\ddist}x| \leq \varepsilon_{k} \defeq \sqrt{\frac{2\exx_{\ddist}x^{2}\log(\delta^{-1})}{k}}
\end{align*}
on an event with probability no less than $1-2\delta$, over the draw of the $k$-sized sub-sample. Using this, we shift the remaining data points as $x_{i}^{\prime} \defeq x_{i}-\overbar{x}_{\trunc}$. Note that the second moment of this data is bounded as follows:
\begin{align*}
\exx_{\ddist}(x^{\prime})^{2} & = \exx_{\ddist}\left( x^{\prime} - \exx_{\ddist}x^{\prime} \right)^{2} + \left(\exx_{\ddist}x^{\prime}\right)^{2}\\
& = \exx_{\ddist}\left( (x-\overbar{x}_{\trunc}) - \exx_{\ddist}(x-\overbar{x}_{\trunc}) \right)^{2} + \left(\exx_{\ddist}(x-\overbar{x}_{\trunc})\right)^{2}\\
& = \exx_{\ddist}\left( x - \exx_{\ddist}x \right)^{2} + (\overbar{x}_{\trunc} - \exx_{\ddist}x)^{2}\\
& \leq \vaa_{\ddist} x + \varepsilon_{k}^{2}.
\end{align*}
Passing these shifted points through (\ref{eqn:defn_est_bernoulli}) with analogous second moment bounds used for scaling, we have
\begin{align}
\label{eqn:xm_centered_main}
\xhat^{\prime} = \frac{s}{(n-k)} \sum_{i=k+1}^{n}\trunc\left(\frac{x_{i}^{\prime}}{s}\right), \text{ where } s^{2} = \frac{(n-k)(\vaa_{\ddist} x + \varepsilon_{k}^{2})}{2\log(\delta^{-1})}.
\end{align}
Shifting the resulting output back to the original location by adding 
and shifting $\xhat^{\prime}$ back to the original location by adding $\overbar{x}_{\trunc}$, conditioned on $\overbar{x}_{\trunc}$, we have by (\ref{eqn:bound_new_1d}) again that
\begin{align*}
|(\xhat^{\prime}+\overbar{x}_{\trunc}) - \exx_{\ddist}x| = |\xhat - \exx_{\ddist}(x-\overbar{x}_{\trunc})| \leq \sqrt{\frac{2(\vaa_{\ddist} x + \varepsilon_{k}^{2})\log(\delta^{-1})}{n-k}}
\end{align*}
with probability no less than $1-2\delta$ over the draw of the remaining $n-k$ points. Defining the centered estimator as $\xhat = \xhat^{\prime} + \overbar{x}_{\trunc}$, and taking a union bound over the two ``good events'' on the independent sample subsets, we may thus conclude that
\begin{align}\label{eqn:bound_new_1d_centered}
\prr\left\{ |\xhat - \exx_{\ddist}x| > \varepsilon \right\} \leq 4 \exp\left( \frac{-(n-k)\varepsilon^{2}}{2(\vaa_{\ddist} x + \varepsilon_{k}^{2})} \right)
\end{align}
where probability is over the draw of the full $n$-sized sample. While one takes a hit in terms of the sample size, the variance works to combat sensitivity to the distribution location.

\section{PAC-Bayesian bounds for heavy-tailed data}\label{sec:bounds}

An import and influential paper due to D.~McAllester gave the following theorem as a motivating result. For clarity to the reader, we give a slightly modified version of his result.

\begin{thm}[\citet{mcallester1999a}, Preliminary Theorem 2]\label{thm:mcallester1999_prethm2}
Let $\prior$ be a prior probability distribution over $\HH$, assumed countable, and to be such that $\prior(h) > 0$ for all $h \in \HH$. Consider the pattern recognition task with $\zz = (\xx,y) \in \XX \times \{-1,1\}$, and the classification error $l(h;\zz) = I\{h(\xx) \neq y\}$. Then with probability no less than $1-\delta$, for any choice of $h \in \HH$, we have
\begin{align*}
R(h) \leq \frac{1}{n} \sum_{i=1}^{n} l(h;\zz_{i}) + \sqrt{\frac{\log\left(1/\prior(h)\right) + \log\left(1/\delta\right)}{2n}}
\end{align*}
\end{thm}
\begin{proof}
For clean notation, denote the empirical risk as
\begin{align*}
\widehat{R}(h) = \frac{1}{n} \sum_{i=1}^{n} l(h;\zz_{i}), \qquad h \in \HH.
\end{align*}
Using a classical Chernoff bound specialized to the case of Bernoulli observations (Lemma \ref{lem:chernoff_tight}), we have that for any $h \in \HH$, it holds that
\begin{align*}
\prr\left\{ R(h) - \widehat{R}(h) > \varepsilon \right\} \leq \exp\left(-2n\varepsilon^{2}\right).
\end{align*}
Rearranging terms, it follows immediately that with probability no less than $1-\prior(h)\,\delta$, we have
\begin{align*}
R(h) - \widehat{R}(h) \leq \varepsilon^{\ast}(h) \defeq \sqrt{\frac{\log(1/\prior(h))+\log(1/\delta)}{2n}}.
\end{align*}
The desired result follows from a union bound:
\begin{align*}
\prr\left\{ \exists\,h\in\HH \text{ s.t.~} R(h)-\widehat{R}(h) > \varepsilon^{\ast}(h) \right\} & \leq \prr \bigcup_{h \in \HH} \left\{ R(h)-\widehat{R}(h) > \varepsilon^{\ast}(h) \right\}\\
& \leq \sum_{h \in \HH} \prr\left\{ R(h)-\widehat{R}(h) > \varepsilon^{\ast}(h) \right\}\\
& \leq \sum_{h \in \HH}\prior(h)\delta\\
& = \delta.
\end{align*}
The event on the left-hand side of the above inequality is precisely that of the hypothesis, namely the ``bad event'' on which the sample is such that the risk $R(h)$ exceeds the given bound for \textit{some} candidate $h \in \HH$.
\end{proof}
One quick glance at the proof of this theorem shows that the bounded nature of the observations plays a crucial role in deriving excess risk bounds of the above form, as it is used to obtain concentration inequalities for the empirical risk about the true risk. While analogous concentration inequalities hold under slightly weaker assumptions, when considering the potentially heavy-tailed setting, one simply cannot guarantee that empirical risk is tightly concentrated about the true risk, which prevents direct extensions of such theorems. With this in mind, we take a different approach, that does not require the empirical mean to be well-concentrated.

\paragraph{Our motivating pre-theorem}

The basic idea of our approach is very simple: instead of using the sample mean, bound the off-sample risk using a more robust estimator which is easy to compute directly, and which allows risk bounds even under unbounded, potentially heavy-tailed losses. Define a new approximation of the risk by
\begin{align}\label{eqn:defn_risk_approx}
\riskhat_{\trunc}(h) \defeq \frac{s}{n} \sum_{i=1}^{n} \trunc\left(\frac{l(h;\zz_{i})}{s}\right),
\end{align}
for $s>0$. Note that this is just a direct application of the robust estimator defined in (\ref{eqn:defn_est_bernoulli}) to the case of a loss which depends on the choice of candidate $h \in \HH$. As a motivating result, we basically re-prove McAllester's result (Theorem \ref{thm:mcallester1999_prethm2}) under much weaker assumptions on the loss, using the statistical properties of the new risk estimator (\ref{eqn:defn_risk_approx}), rather than relying on classical Chernoff inequalities.

\begin{thm}[Pre-theorem]\label{thm:PBPB_motivation}
Let $\prior$ be a prior probability distribution over $\HH$, assumed countable. Assume that $\prior(h)>0$ for all $h \in \HH$, and that $m_{2}(h) \defeq \exx l(h;\zz)^{2} < \infty$ for all $h \in \HH$. Setting the scale in (\ref{eqn:defn_risk_approx}) to $s_{h}^{2} = n\,m_{2}(h) / 2\log(\delta^{-1})$, then with probability no less than $1-2\delta$, for any choice of $h \in \HH$, we have
\begin{align*}
R(h) \leq \riskhat_{\trunc}(h) + \sqrt{\frac{2m_{2}(h)\left(\log(1/\prior(h)) + \log(1/\delta)\right)}{n}}.
\end{align*}
\end{thm}
\begin{proof}
We start by making use of the pointwise deviation bound given in Proposition \ref{prop:devdb_pointwise}, which tells us that with $(1-2\delta)$ high probability
\begin{align*}
R(h) \leq \frac{s}{n}\sum_{i=1}^{n} \trunc\left(\frac{l(h;\zz_{i})}{s}\right) + \sqrt{\frac{2m_{2}(h)\log(\delta^{-1})}{n}}
\end{align*}
for any pre-fixed $h \in \HH$. Replacing $\delta$ with $\prior(h)\delta$ gives the key error level
\begin{align*}
\varepsilon^{\ast}(h) \defeq \sqrt{\frac{2m_{2}(h)\left(\log(1/\prior(h)) + \log(1/\delta)\right)}{n}},
\end{align*}
and using the union bound argument in the proof of Theorem \ref{thm:mcallester1999_prethm2}, we have
\begin{align*}
\prr\left\{ \exists\,h \in \HH \text{ s.t.~} R(h) - \riskhat_{\trunc}(h) > \varepsilon^{\ast}(h) \right\} \leq 2\delta.
\end{align*}
\end{proof}
\begin{rmk}
We note that all quantities on the right-hand side of Theorem \ref{thm:PBPB_motivation} are easily computed based on the sample, except for the second moment $m_{2}$, which in practice must be replaced with an empirical estimate. With an empirical estimate of $m_{2}$ in place, the upper bound can easily be used to derive a learning algorithm.
\end{rmk}

\paragraph{Uncountable model case}

Next we extend the previous motivating theorem to a more general result on a potentially uncountable $\HH$, using stochastic learning algorithms, as has become standard in the PAC-Bayes literature. We need a few technical conditions, listed below:
\begin{enumerate}
\item Bounds on lower-order moments. For all $h \in \HH$, we require $\exx_{\ddist}l(h;\zz)^{2} \leq \bdmnt < \infty$, $\exx_{\ddist}l(h;\zz)^{3} \leq \bdthird < \infty$.
\item Bounds on the risk. For all $h \in \HH$, we require $\risk(h) \leq \sqrt{n\bdmnt/(4\log(\delta^{-1}))}$.
\item Large enough confidence. We require $\delta \leq \exp(-1/9) \approx 0.89$.
\end{enumerate}
These conditions are quite reasonable, and easily realized under heavy-tailed data, with just lower-order moment assumptions on $\ddist$ and say a compact class $\HH$. The new terms that appear in our bounds that do no appear in previous works are $\emgrisk_{\post,\trunc} \defeq \exx_{\post}\riskhat_{\trunc}$ and $\prior_{n}^{\ast}(\HH) = \exx_{\prior}\exp(\sqrt{n}(\risk - \riskhat_{\trunc}))/\exx_{\prior}\exp(\risk - \riskhat_{\trunc})$. The former is the expectation of the proposed robust estimator with respect to posterior $\post$, and the latter is a term that depends directly on the quality of the prior $\prior$.

\begin{thm}\label{thm:PBPB_uncountable}
Let $\prior$ be a prior distribution on model $\HH$. Let the three assumptions listed above hold. Setting the scale in (\ref{eqn:defn_risk_approx}) to $s^{2} = n\,\bdmnt / 2\log(\delta^{-1})$, then with probability no greater than $1-\delta$ over the random draw of the sample, it holds that
\begin{align*}
\grisk_{\post} \leq \emgrisk_{\post,\trunc} + \frac{1}{\sqrt{n}}\left( \KL(\post;\prior) + \frac{\log(8\pi \bdmnt \delta^{-2})}{2} + \bdmnt + \prior_{n}^{\ast}(\HH) - 1 \right) + O\left(\frac{1}{n}\right)
\end{align*}
for any choice of probability distribution $\post$ on $\HH$, since $\grisk_{\post} < \infty$ by assumption.
\end{thm}
\begin{rmk}
As is evident from the statement of Theorem \ref{thm:PBPB_uncountable}, the convergence rate is clear for all terms but $\prior_{n}^{\ast}(\HH)/\sqrt{n}$. In our proof, we use a modified version of the elegant and now-standard strategy formulated by \citet{begin2016a}. A glance at the proof shows that under this strategy, there is essentially no way to avoid dependence on $\prior_{n}^{\ast}(\HH)$. Since the random variable $\risk-\riskhat_{\trunc}$ is bounded over the random draw of the sample and $h \sim \prior$, the bounds still hold and are non-trivial. That said, $\prior_{n}^{\ast}(\HH)$ may indeed increase as $n \to \infty$, potentially spoiling the $\sqrt{n}$ rate, and even consistency in the worst case. Clearly $\prior_{n}^{\ast}(\HH)$ presents no troubles if $\risk - \riskhat_{\trunc} \leq 0$ on a high-probability event, but note that this essentially amounts to asking for a prior that on average realizes bounds that are better than we can guarantee for \emph{any} posterior though the above analysis. Such a prior may indeed exist, but if it were known, then that would eliminate the need for doing any learning at all. If the deviations $\risk-\riskhat_{\trunc}$ are truly sub-Gaussian \citep{boucheron2013a}, then the $\sqrt{n}$ rate can be easily obtained. However, impossibility results from \citet{devroye2016a} suggest that under just a few finite moment assumptions, such an estimator cannot be constructed. As such, here we see a clear limitation of the established PAC-Bayes analytical framework under potentially heavy-tailed data. Since the change of measures step in the proof is fundamental to the basic argument, it appears that concessions will have to be made, either in the form of slower rates, deviations larger than the relative entropy, or weaker dependence on $1/\delta$.
\end{rmk}
\begin{rmk}
Note that while in its tightest form, the above bound requires knowledge of $\exx_{\ddist}l(h;\zz)^{2}$, we may set $s>0$ used to define $\riskhat_{\trunc}$ using any valid upper bound $\bdmnt$, under which the above bound still holds as-is, using known quantities. Furthermore, for reference the content of the $O(1/n)$ term in the above bound takes the form
\begin{align*}
\frac{1}{n}\left(2\sqrt{\bdvar\log(\delta^{-1})} + \frac{\bdthird\log(\delta^{-1})}{3\bdmnt \sqrt{n}} \right)
\end{align*}
where $V$ is an upper bound on the variance $\vaa_{\ddist} l(h;\zz) \leq V < \infty$ over $h \in \HH$.
\end{rmk}
\begin{proof}[Proof of Theorem \ref{thm:PBPB_uncountable}]
To begin, let us recall a useful ``change of measures'' inequality,\footnote{There are other very closely related approaches to this proof. See \citet{tolstikhin2013a,begin2016a} for some recent examples. Furthermore, we note that the key facts used here are also present in \citet{catoni2007a}.} which can be immediately derived from our proof of Theorem \ref{thm:main_PB_identity}. In particular, recall from identity (\ref{eqn:main_PB_identity_linearity}) that given some prior $\prior$ and constructing $\prior^{\ast}$ such that almost everywhere $[\prior]$ one has
\begin{align*}
\left(\frac{d\prior^{\ast}}{d\post}\right)(h) = \frac{\exp(\varphi(h))}{\exx_{\prior}\exp(\varphi)},
\end{align*}
it follows that
\begin{align*}
\KL(\post;\prior^{\ast}) & = \exx_{\post}\left(\log\frac{d\post}{d\prior} + \log\exx_{\prior}\exp(\varphi) - \varphi\right)\\
& = \KL(\post;\prior) + \log\exx_{\prior}\exp(\varphi) - \exx_{\post}\varphi
\end{align*}
whenever $\exx_{\post}\varphi < \infty$. In the case where $\exx_{\post}\varphi = \infty$, upper bounds are of course meaningless. Re-arranging, observe that since $\KL(\post;\prior^{\ast}) \geq 0$, it follows that
\begin{align}\label{eqn:KL_change_of_measure}
\exx_{\post}\varphi \leq \KL(\post;\prior) + \log\exx_{\prior}\exp(\varphi).
\end{align}
This inequality given in (\ref{eqn:KL_change_of_measure}) is deterministic, holds for any choice of $\post$, and is a standard technical tool in deriving PAC-Bayes bounds.

We shall introduce a minor modification to this now-standard strategy in order to make the subsequent results more lucid. Instead of $\prior^{\ast}$ as just characterized above, define $\prior^{\ast}_{n}$ such that almost surely $[\prior]$, we have
\begin{align*}
\left(\frac{d\prior_{n}^{\ast}}{d\post}\right)(h) = g(h) \defeq \frac{\exp(\varphi(h))}{\exx_{\prior}\exp(\varphi/c_{n})},
\end{align*}
where $1 \leq c_{n} < \infty$ is a function of the sample size $n$, which increases monotonically as $c_{n} \uparrow \infty$ when $n \to \infty$ (e.g., setting $c_{n} = \sqrt{n}$). To explicitly construct such a measure, one can define it by $\prior_{n}^{\ast}(A) \defeq  \int_{A} g \, d\prior$, for all $A \subset \AAcal$, where $(\HH,\AAcal)$ is our measurable space of interest. In this paper, we always\footnote{We will only be using $\varphi \propto \risk-\riskhat_{\trunc}$, so this statement holds via $\risk \geq 0$ and $\|\risk_{\trunc}\|_{\infty} < \infty$.} have $\varphi > -\infty$, implying that $\exx_{\prior} \exp(\varphi) > 0$. Also by assumption, since $\risk$ is bounded over $h \in \HH$, we have $\exx_{\prior} \exp(\varphi) < \infty$, which in turn implies
\begin{align*}
0 < \prior_{n}^{\ast}(\HH) = \frac{\exx_{\prior}\exp(\varphi)}{\exx_{\prior}\exp(\varphi/c_{n})} < \infty,
\end{align*}
and so $\prior_{n}^{\ast}$ is a finite measure. Note however that both $\prior_{n}^{\ast}(\HH) > 1$ and $\prior_{n}^{\ast}(\HH) < 1$ are possible, so in general $\prior_{n}^{\ast}$ need not be a probability measure. By construction, we have $\prior_{n}^{\ast} \ll \prior$. Since $\varphi(h) > -\infty$ for all $h \in \HH$, we have that $g > 0$ and thus the measurability of $g$ implies the measurability of $1/g$. Using the chain rule (Lemma \ref{lem:chain_rule}), it follows that for any $A \in \AAcal$,
\begin{align*}
\int_{A} \left(\frac{1}{g}\right) \, d\prior_{n}^{\ast} = \int_{A} \left(\frac{1}{g}\right) \left(g\right) \, d\prior = \prior(A).
\end{align*}
As such, we have $\prior \ll \prior_{n}^{\ast}$, and by the Radon-Nikonym theorem, we may write $1/g = d\prior/d\prior_{n}^{\ast}$ since such a function is unique almost everywhere $[\prior_{n}^{\ast}]$. As long as $\post \ll \prior$, which in turn implies $\post \ll \prior_{n}^{\ast}$, so that with use of the chain rule and Radon-Nikodym, we have
\begin{align*}
\int_{A} \left(\frac{d\post}{d\prior}\right)\left(\frac{1}{g}\right) \, d\prior_{n}^{\ast} = \int_{A} \left(\frac{d\post}{d\prior}\right)\left(\frac{1}{g}\right) g \, d\prior = \post(A) = \int_{A} \left(\frac{d\post}{d\prior_{n}^{\ast}}\right) \, d\prior_{n}^{\ast}.
\end{align*}
Taking the two ends of this string of equalities, by Radon-Nikodym it holds that
\begin{align*}
\frac{d\post}{d\prior} \frac{d\prior}{\prior_{n}^{\ast}} = \frac{d\post}{d\prior_{n}^{\ast}}
\end{align*}
a.e.~$[\prior_{n}^{\ast}]$, and thus a.e.~$[\post]$ as well. Following the argument of Theorem \ref{thm:main_PB_identity}, we have that
\begin{align*}
\KL(\post;\prior_{n}^{\ast}) = \KL(\post;\prior) + \log\exx_{\prior}\exp(\varphi/c_{n}) - \exx_{\post}\varphi.
\end{align*}

The tradeoff for using $\prior_{n}^{\ast}$ which need not be a probability comes in deriving a lower bound on $\KL(\prior;\prior_{n}^{\ast})$. In Lemma \ref{lem:KL_nonneg} we showed how the relative entropy between probability measures is non-negative. Non-negativity does not necessarily hold for general measures, but analogous lower bounds can be readily derived for our special case as
\begin{align*}
\KL(\post;\prior_{n}^{\ast}) = \exx_{\post}\log\frac{d\post}{d\prior_{n}^{\ast}} = \exx_{\prior_{n}^{\ast}}\frac{d\post}{d\prior_{n}^{\ast}}\log\frac{d\post}{d\prior_{n}^{\ast}} \geq \exx_{\prior_{n}^{\ast}}\left( \frac{d\post}{d\prior_{n}^{\ast}} - 1 \right) = 1 - \prior_{n}^{\ast}(\HH),
\end{align*}
where the last inequality uses the fact that $\post$ is a probability and $\post(A) = \int_{A}(d\post/d\prior_{n}^{\ast}) \, d\prior_{n}^{\ast}$ for all $A \in \AAcal$. Taking this with our decomposition of $\KL(\post;\prior_{n}^{\ast})$, we have
\begin{align}\label{eqn:KL_modified_change_of_measure}
\exx_{\post}\varphi \leq \KL(\post;\prior) + \log\exx_{\prior}\exp\left(\varphi/c_{n}\right) - 1 + \prior_{n}^{\ast}(\HH),
\end{align}
which amounts to a revised inequality based on change of measures, analogous to (\ref{eqn:KL_change_of_measure}).

To keep notation clean, write
\begin{align*}
X(h) & \defeq R(h) - \frac{s}{n} \sum_{i=1}^{n} \trunc\left(\frac{l(h;\zz_{i})}{s}\right) = R(h) - \riskhat_{\trunc}(h)\\
m_{2}(h) & \defeq \exx_{\ddist}l(h;\zz)^{2}\\
v(h) & \defeq \exx_{\ddist}(l(h;\zz)-\risk(h))^{2}
\end{align*}
Noting that $X(h)$ is random with dependence on the sample, via Markov's inequality we have
\begin{align}\label{eqn:PBPB_uncountable_markov}
\exx_{\prior} e^{X} \leq \frac{\exx_{n} \exx_{\prior} e^{X}}{\delta},
\end{align}
with probability no less than $1-\delta$. Here probability and $\exx_{n}$ are with respect to the sample. Since $\riskhat_{\trunc}$ is bounded, as long as $\exx_{\post}R < \infty$, we have $\exx_{\post}X < \infty$, which lets us use the change of measures inequality in a meaningful way. Now for $c_{n} > 0$, observe that we have
\begin{align*}
c_{n}\exx_{\post}X = \exx_{\post}c_{n}X & \leq \KL(\post;\prior) + \log\exx_{\prior}\exp\left(X\right) - 1 + \prior_{n}^{\ast}(\HH)\\
&\leq \KL(\post;\prior) + \log(\delta^{-1}) + \log \exx_{n}\exx_{\prior}\exp\left(X\right) - 1 + \prior_{n}^{\ast}(\HH)\\
& = \KL(\post;\prior) + \log(\delta^{-1}) + \log \exx_{\prior}\exx_{n}\exp\left(X\right) - 1 + \prior_{n}^{\ast}(\HH)
\end{align*}
with probability no less than $1-\delta$. The first inequality follows from modified change of measures (\ref{eqn:KL_modified_change_of_measure}), the second inequality follows from (\ref{eqn:PBPB_uncountable_markov}), and the final interchange of integration operations is valid using Fubini's theorem \citep{ash2000a}. Note that the $1-\delta$ ``good event'' depends only on $\prior$ (fixed in advance) and not $\post$. Thus, the above inequality holds on the good event, uniformly in $\post$.

It remains to bound $\exx_{n}\exp(cX)$, for an arbitrary constant $c>0$ (here we will have $c=1$). Start by breaking up the one-sided deviations as
\begin{align*}
X = \risk - \riskhat_{\trunc} = \left(\risk - \exx_{n}\riskhat_{\trunc}\right) + \left(\exx_{n}\riskhat_{\trunc} - \riskhat_{\trunc}\right),
\end{align*}
writing $X_{(1)} \defeq \risk - \exx_{n}\riskhat_{\trunc}$ and $X_{(2)} \defeq \exx_{n}\riskhat_{\trunc} - \riskhat_{\trunc}$ for convenience. We will take the terms $X_{(1)}$ and $X_{(2)}$ one at a time. First, note that the function $\trunc$ can be written
\begin{align}\label{eqn:trunc_non_piecewise}
\trunc(u) = \left(u - \frac{u^{3}}{6}\right)\left( I\{u \leq \sqrt{2}\} - I\{u < -\sqrt{2}\} \right) + \frac{2\sqrt{2}}{3}\left(1 - I\{u \leq \sqrt{2}\} - I\{u < -\sqrt{2}\}\right).
\end{align}
Again for notational simplicity, write $L = l(h;\zz)$ and $L_{i} = l(h;\zz_{i})$, $i \in [n]$, where $h \in \HH$ is arbitrary.  Write $\EE_{i}^{+} \defeq \left\{L_{i} \leq s\sqrt{2}\right\}$ and $\EE_{i}^{-} \defeq \left\{L_{i} < -s\sqrt{2}\right\}$. We are assuming non-negative losses, so that $L \geq 0$. This means that $I(\EE_{i}^{-}) = 0$ and $\prr\EE_{i}^{-} = 0$. We use this, as well as $1-\prr\EE_{i}^{+} \geq 0$, in addition to (\ref{eqn:trunc_non_piecewise}) in order to bound the expectation of our estimator $\riskhat_{\trunc}$ from below, as follows.
\begin{align*}
\exx_{n}\riskhat_{\trunc} & = \frac{s}{n} \sum_{i=1}^{n} \exx_{\ddist} \trunc\left(\frac{L_{i}}{s}\right)\\
& = \frac{s}{n} \sum_{i=1}^{n} \left[ \exx_{\ddist} \left( \frac{L_{i}}{s} - \frac{L_{i}^{3}}{6s^{3}} \right)\left( I(\EE_{i}^{+}) - I(\EE_{i}^{-}) \right) + \frac{2\sqrt{2}}{3}\left( 1 - \prr\EE_{i}^{+} - \prr\EE_{i}^{-} \right) \right]\\
& \geq \frac{s}{n} \sum_{i=1}^{n} \exx_{\ddist} \left( \frac{L_{i}}{s} - \frac{L_{i}^{3}}{6s^{3}} \right)I(\EE_{i}^{+})\\
& = \exx_{\ddist}L\,I\{ L \leq s\sqrt{2}\} - \frac{1}{6s^{2}}\exx_{\ddist}L^{3}I\{L \leq s\sqrt{2}\}\\
& = \risk - \exx_{\ddist} L\,I\{X > s\sqrt{2}\} - \frac{1}{6s^{2}}\exx_{\ddist}L^{3}I\{L \leq s\sqrt{2}\}.
\end{align*}
By assumption, we have $\exx_{\ddist} L^{3} I\{L \leq s\sqrt{2}\} \leq \exx L^{3} \leq \bdthird < \infty$, implying that this lower bound is non-trivial. Next we obtain a one-sided bound on the tails of the loss by
\begin{align*}
\prr\left\{ L > s\sqrt{2} \right\} & = \prr\left\{ L-\risk > s\sqrt{2}-\risk \right\}\\
& \leq \prr\left\{ |L-\risk| > s\sqrt{2}-\risk \right\}\\
& \leq \frac{\exx_{\ddist}|L-\risk|^{2}}{\left(s\sqrt{2}-\risk\right)^{2}}.
\end{align*}
Note that the first inequality makes use of $s\sqrt{2} > \risk$, which is implied by the bounds assumed on $\risk$, namely that $1/2 \geq \risk\sqrt{\log(\delta^{-1})/(n\bdmnt)}$.

Returning to the lower bound on $\riskhat_{\trunc}$, using H\"{o}lder's inequality in conjunction with the tail bound we just obtained, we get an upper bound in the form of
\begin{align*}
\exx_{\ddist}L\,I\{ L > s\sqrt{2}\} & = \exx_{\ddist}|L\,I\{ L > s\sqrt{2}\}|\\
& \leq \sqrt{\exx_{\ddist}L^{2}\prr\{L > s\sqrt{2}\}}\\
& \leq \sqrt{\frac{\exx_{\ddist}L^{2}\exx_{\ddist}|L-\risk|^{2}}{\left(s\sqrt{2}-\risk\right)^{2}}}.
\end{align*}
This means we can now say
\begin{align*}
\exx_{n} \riskhat_{\trunc} \geq \risk - \sqrt{\frac{\exx_{\ddist}L^{2}\exx_{\ddist}|L-\risk|^{2}}{\left(s\sqrt{2}-\risk\right)^{2}}} - \frac{1}{6s^{2}}\exx_{\ddist}L^{3}I\{L \leq s\sqrt{2}\},
\end{align*}
which re-arranged and written more succinctly gives us
\begin{align*}
X_{(1)}(h) & \leq \sqrt{\frac{m_{2}(h)v(h)}{\left(s\sqrt{2}-\risk\right)^{2}}} + \frac{\exx_{\ddist}|l(h;\zz)|^{3}}{6s^{2}}\\
& \leq \sqrt{\frac{\bdmnt \bdvar}{\left(s\sqrt{2}-\risk\right)^{2}}} + \frac{\bdthird}{6s^{2}}\\
& = \sqrt{\frac{\bdvar\log(\delta^{-1})}{n\left(1-R\sqrt{\log(\delta^{-1})/(n\bdmnt)}\right)^{2}}} + \frac{\bdthird\log(\delta^{-1})}{3\bdmnt n}\\
& \leq B_{(1)}\\
& \defeq 2\sqrt{\frac{\bdvar\log(\delta^{-1})}{n}} + \frac{\bdthird\log(\delta^{-1})}{3\bdmnt n}
\end{align*}
as desired. The final inequality uses the assumed bound on $\risk$. Note that this is a deterministic bound, in that it is free of both the choice of $h$ (i.e., random draw from $\prior$ or $\post$) and the sample, which we are integrating over.

Next, we look at the remaining deviations $X_{(2)} = \exx_{n}\riskhat_{\trunc} - \riskhat_{\trunc}$. Writing $Y_{i} \defeq (s/n)\trunc(L_{i}/s)$, we have $X_{(2)} = \sum_{i=1}^{n}(\exx Y_{i} - Y_{i})$. Since $0 \leq \trunc(u) \leq 2\sqrt{2}/3$ for $u \geq 0$, and $L \geq 0$, we have that $0 \leq Y_{i} \leq 2\sqrt{2}s/(3n)$. It follows from Hoeffding's inequality that for all $\epsilon > 0$, we have
\begin{align}\nonumber
\prr\left\{ X_{(2)} > \epsilon \right\} & \leq \exp\left(\frac{-2\epsilon^{2}}{n(2\sqrt{2}s/(3n))^{2}}\right)\\
\label{eqn:partial_bound_hoeffding}
& = \exp\left(\frac{-9\epsilon^{2}\log(\delta^{-1})}{2\bdmnt}\right).
\end{align}
Note that this bound does not depend on the setting of $\delta \in (0,1)$, which is fixed in advance. Also note that while we are dealing with the sum of bounded, independent random variables, the scaling factor $s \propto \sqrt{n}$ makes it such that these deviations converge to some potentially non-zero constant in the $n \to \infty$ limit, which is why $n$ does not appear in the exponential on the right-hand side.

In any case, we can still readily use these sub-Gaussian tail bounds to control the expectation. Using the classic identity relating the expectation to the tails of a distribution,
\begin{align}
\nonumber
\exx_{n} \exp\left(cX_{(2)}\right) & = \int_{0}^{\infty} \prr\left\{ \exp\left(cX_{(2)}\right) > \epsilon \right\} \, d\epsilon\\
\label{eqn:PBPB_uncountable_link}
& = \int_{-\infty}^{\infty} \prr\left\{ \exp\left(cX_{(2)}\right) > \exp(\epsilon) \right\} \exp(\epsilon) \, d\epsilon
\end{align}
where the second equality follows using integration by substitution. The right-hand side of (\ref{eqn:PBPB_uncountable_link}) is readily controlled as follows. Using (\ref{eqn:partial_bound_hoeffding}) above, we have
\begin{align*}
\prr\left\{ \exp\left(cX_{(2)}\right) > \exp(\epsilon) \right\} = \prr\left\{ X_{(2)} > \epsilon/c \right\} \leq \exp\left(\frac{-\epsilon^{2}}{2\sigma^{2}}\right)
\end{align*}
where we have set $\sigma^{2} \defeq c^{2}\bdmnt/(9\log(\delta^{-1}))$. The key bound of interest can be compactly written as
\begin{align*}
\exx_{n} \exp\left(cX_{(2)}\right) & \leq 2\int_{-\infty}^{\infty} \exp\left( -\frac{\epsilon^{2}}{2\sigma^{2}} + \epsilon \right) \, d\epsilon\\
& = 2\int_{-\infty}^{\infty} \exp\left( -\frac{1}{2\sigma^{2}}\left(\epsilon-\sigma^{2}\right)^{2} + \frac{\sigma^{2}}{2} \right) \, d\epsilon\\
& = 2\sqrt{2\pi}\sigma \exp\left(\frac{\sigma^{2}}{2}\right) \int_{-\infty}^{\infty} \frac{1}{\sqrt{2\pi}\sigma} \exp\left(-\frac{1}{2\sigma^{2}}\left(\epsilon-\sigma^{2}\right)^{2}\right) \, d\epsilon\\
& = 2\sqrt{2\pi}\sigma \exp\left(\frac{\sigma^{2}}{2}\right).
\end{align*}
Note that the first equality uses the usual ``complete the square'' identity, and the rest follows from basic properties of the Gaussian integral. Filling in the definition of $\sigma$, we have
\begin{align*}
\exx_{n} \exp\left(cX_{(2)}\right) \leq 2\sqrt{2\pi} \left(c\sqrt{\frac{\bdmnt}{9\log(\delta^{-1})}}\right)\exp\left( \frac{c^{2}\bdmnt}{9\log(\delta^{-1})} \right).
\end{align*}
The right-hand side of this inequality is free of the choice of $h \in \HH$, and thus taking expectation with respect to $\prior$ yields the same bound, i.e., the same bound holds for $\exx_{\prior}\exx_{n}\exp(cX_{(2)})$. Taking the log of this upper bound, we thus may conclude that
\begin{align*}
\log \exx_{\prior}\exx_{n} \exp(cX_{(2)}) & \leq \frac{1}{2}\left[\log(8\pi \bdmnt c^{2})-\log(9\log(\delta^{-1}))\right] + \frac{c^{2}\bdmnt}{9\log(\delta^{-1})}\\
& \leq \frac{1}{2}\log(8\pi \bdmnt c^{2}) + c^{2}\bdmnt
\end{align*}
on an event of probability no less than $1-\delta$. The latter inequality uses $\delta \leq \exp(-1/9)$. For the result of interest here, we can let $c=1$. 

Finally, going back to the bound on $c_{n}\exx_{\post}X$, we can control the key term by
\begin{align*}
\log \exx_{\prior}\exx_{n}\exp\left(X\right) & = \log \exx_{\prior}\exx_{n}\exp\left(X_{(1)}+X_{(2)}\right)\\
& = \log \exx_{\prior}\left[\exp\left(X_{(1)}\right)\exx_{n}\exp\left(X_{(2)}\right)\right]\\
& \leq B_{(1)} + \log\exx_{\prior}\exx_{n}\exp\left(X_{(2)}\right)\\
& \leq B_{(1)} + \frac{1}{2}\log(8\pi \bdmnt c^{2}) + c^{2}\bdmnt.
\end{align*}
Setting $c_{n} = \sqrt{n}$, we have
\begin{align*}
\sqrt{n} \exx_{\post} X & \leq \KL(\post;\prior) + \log(\delta^{-1}) + \log \exx_{\prior}\exx_{n}\exp\left(X\right) - 1 + \prior_{n}^{\ast}(\HH)\\
& \leq \KL(\post;\prior) + \log(\delta^{-1}) + 2\sqrt{\frac{\bdvar\log(\delta^{-1})}{n}} + \frac{\bdthird\log(\delta^{-1})}{3\bdmnt n} + \frac{\log(8\pi \bdmnt)}{2} + \bdmnt - 1 + \prior_{n}^{\ast}(\HH).
\end{align*}
Dividing both sides by $\sqrt{n}$ yields the desired result.
\end{proof}

As a principled approach to deriving stochastic learning algorithms, one naturally considers the choice of posterior $\post$ in Theorem \ref{thm:PBPB_uncountable} that minimizes the upper bound. This is typically referred to as the optimal Gibbs posterior \citep{germain2016a}, and takes a form which is easily characterized, as we prove in the following proposition.
\begin{prop}[Robust optimal Gibbs posterior]\label{prop:optimal_gibbs_robust}
The upper bound of Theorem \ref{thm:PBPB_uncountable} is optimized by a data-dependent posterior distribution $\gibbs$, defined in terms of its density function with respect to the prior $\prior$ as
\begin{align*}
\left(\frac{d\gibbs}{d\prior}\right)\left(h\right) = \frac{\exp\left(-\sqrt{n}\riskhat_{\trunc}(h)\right)}{\exx_{\prior}\exp\left(-\sqrt{n}\riskhat_{\trunc}\right)}.
\end{align*}
Furthermore, the risk bound under the optimal Gibbs posterior takes the form
\begin{align*}
\grisk_{\gibbs} \leq \frac{1}{\sqrt{n}}\left( \log\exx_{\prior}\exp\left(\sqrt{n}\riskhat_{\trunc}\right) + \frac{\log(8\pi \bdmnt\delta^{-1})}{2} + \bdmnt + \prior_{n}^{\ast}(\HH) - 1 \right) + O\left(\frac{1}{n}\right)
\end{align*}
with probability no less than $1-\delta$ over the draw of the sample.
\end{prop}
\begin{proof}[Proof of Proposition \ref{prop:optimal_gibbs_robust}]
To keep the notation clean, write $X = X(h) = -\sqrt{n}\riskhat_{\trunc}(h)$. Similar to the proof of Theorem \ref{thm:main_PB_identity}, we have
\begin{align*}
\KL(\post;\gibbs) & = \exx_{\post}\log\left(\frac{d\post}{d\gibbs}\right)\\
& = \exx_{\post} \log \left(\frac{d\post}{d\prior}\frac{d\prior}{d\gibbs}\right)\\
& = \exx_{\post} \left( \log\frac{d\post}{d\prior} + \log\exx_{\prior}\exp(X) - X \right)\\
& = \KL(\post;\prior) + \log\exx_{\prior}\exp(X) - \exx_{\post} X
\end{align*}
whenever $\exx_{\post} X < \infty$. Using non-negativity of the relative entropy (Lemma \ref{lem:KL_nonneg}), the left-hand side of this chain of equalities is minimized in $\post$ at $\post=\gibbs$. Since $\log\exx_{\prior}\exp(X)$ is free of $\post$, it follows that
\begin{align*}
\gibbs & \in \argmin_{\post} \left( \KL(\post;\prior) + \exx_{\post}(-1)X \right)\\
& = \argmin_{\post} \left( \frac{\KL(\post;\prior)}{\sqrt{n}} + \exx_{\post}\riskhat_{\trunc}(h) \right)\\
& = \argmin_{\post} \left( \frac{\KL(\post;\prior)}{\sqrt{n}} + \exx_{\post}\riskhat_{\trunc}(h) + C \right)
\end{align*}
where $C$ is any term which is constant in $\post$, for example all the terms in the upper bound of Theorem \ref{thm:PBPB_uncountable} besides $\emgrisk_{\post,\trunc} + \KL(\post;\prior)/\sqrt{n}$. This proves the result regarding the form of the new optimal Gibbs posterior.

Evaluating the risk bound under this posterior is straightforward computation. Observe that
\begin{align*}
\KL(\gibbs;\prior) & = \exx_{\gibbs} \log \frac{d\gibbs}{d\prior}\\
& = \exx_{\gibbs}\left( X(h) - \log\exx_{\prior}\exp\left(X\right) \right)\\
& = -\sqrt{n}\exx_{\gibbs}\riskhat_{\trunc} - \log\exx_{\prior} \exp\left(-\sqrt{n}\riskhat_{\trunc}\right)\\
& = \log\exx_{\prior} \exp\left(\sqrt{n}\riskhat_{\trunc}\right) - \sqrt{n}\exx_{\gibbs}\riskhat_{\trunc}.
\end{align*}
Substituting this into the upper bound of Theorem \ref{thm:PBPB_uncountable}, the robust empirical mean estimate terms cancel, and we have
\begin{align*}
\grisk_{\gibbs} & \defeq \exx_{\gibbs}R \leq \frac{1}{\sqrt{n}}\left( \log\exx_{\prior}\exp\left(\sqrt{n}\riskhat_{\trunc}\right) + \frac{\log(8\pi \bdmnt\delta^{-2})}{2} + \bdmnt + \prior_{n}^{\ast}(\HH) - 1 \right) + O\left(\frac{1}{n}\right).
\end{align*}
\end{proof}
\begin{rmk}
The bound in Proposition \ref{prop:optimal_gibbs_robust} achieved by the optimal Gibbs posterior computed based on the data is rather straightforward to interpret. Prior knowledge is reflected explicitly in that a prior $\prior$ which performs better in the sense of smaller $\riskhat_{\trunc}$ leads to a smaller risk bound.
\end{rmk}
\begin{rmk}[Comparison with traditional Gibbs posterior]\label{rmk:compare_traditional_gibbs}
In traditional PAC-Bayes analysis \citep[Equation 8]{germain2016a}, the optimal Gibbs posterior, let us write $\gibbs_{\text{emp}}$, is defined by
\begin{align*}
\left(\frac{d\gibbs_{\text{emp}}}{d\prior}\right)\left(h\right) = \frac{\exp\left(-n\widehat{R}(h)\right)}{\exx_{\prior}\exp\left(-n\widehat{R}\right)}
\end{align*}
where $\widehat{R}(h) = n^{-1}\sum_{i=1}^{n} l(h;\zz_{i})$ is the empirical risk. We have $n\widehat{R}$ and $\sqrt{n}\riskhat_{\trunc}$, but since scaling in the latter case should be done with $s \propto \sqrt{n}$, so in both cases the $1/n$ factor cancels out. In the special case of the negative log-likelihood loss, \citet{germain2016a} demonstrate that the optimal Gibbs posterior coincides with the classical Bayesian posterior. As noted by \citet{alquier2016a}, the optimal Gibbs posterior has shown strong empirical performance in practice, and variational approaches have been proposed as efficient alternatives to more traditional MCMC-based implementations. Comparison of both the computational and learning efficiency of our proposed ``robust Gibbs posterior'' with the traditional Gibbs posterior is a point of significant interest moving forward.
\end{rmk}

\section{Conclusions}\label{sec:conclusions}

The main contribution of this paper was to develop a novel approach to obtaining PAC-Bayesian learning guarantees, which admits deviations with exponential tails under weak moment assumptions on the underlying loss distribution, while still being computationally amenable. In this work, our chief interest was the fundamental problem of obtaining strong guarantees for stochastic learning algorithms which can reflect prior knowledge about the data-generating process, from which we derived a new robust Gibbs posterior. Moving forward, a deeper study of the statistical nature of this new stochastic learning algorithm, as well as computational considerations to be made in practice are of significant interest.

\appendix

\section{Technical appendix}

\subsection{Additional proofs}\label{sec:tech_proofs}

\paragraph{Relative entropy}

% Relative entropy for general probability measures.
Here we recall the basic notions of the relative entropy, or Kullback-Leibler divergence, between two probability distributions. Consider $P$ and $Q$, both defined over a finite space $\Omega$. The relative entropy of $P$ from $Q$ is defined
\begin{align}
\KL(P;Q) \defeq \sum_{\omega \in \Omega} P(\omega) \log \left( \frac{P(\omega)}{Q(\omega)} \right),
\end{align}
where this definition clearly includes the possibility that $\KL(P;Q) = \infty$, which occurs only when $Q$ assigns zero probability to an element that $P$ assigns positive probability to.

More generally, when $\Omega$ is potentially uncountably infinite, consider two probabilities $P$ and $Q$ on the measurable space $(\Omega,\AAcal)$, where $\AAcal$ is an appropriate $\sigma$-algebra.\footnote{A certain degree of measure theory is assumed in this exposition, at approximately the level of the first few chapters of \citet{ash2000a}.} In this case, the relative entropy is defined
\begin{align}
\KL(P;Q) \defeq \int_{\Omega} \log\left( \frac{dP}{dQ} \right) \, dP, \qquad P \ll Q
\end{align}
where $dP/dQ$ denotes the Radon-Nikodym derivative of $P$ with respect to $Q$, typically called the density of $P$ with respect to $Q$. The basic underlying technical assumption, denoted $P \ll Q$, is that $P$ be absolutely continuous with respect to $Q$, meaning that $P(A) = 0$ whenever $Q(A)=0$, for $A \in \AAcal$. In the event that $P \ll Q$ does not hold, by convention we define $\KL(P;Q) \defeq \infty$. Recall that the Radon-Nikodym theorem guarantees that when $P \ll Q$, there exists a measurable function $g \geq 0$ such that
\begin{align*}
P(A) = \int_{A} g \, dQ, \qquad A \in \AAcal.
\end{align*}
This function $g$ is unique in the sense that if there exists another $f$ satisfying the above equality, then $f=g$ almost everywhere $[Q]$. This uniqueness justifies using the notation $dP/dQ$, and calling this function \textit{the} density of $P$ (rather than \textit{a} density of $P$).

\begin{lem}[Chain rule]\label{lem:chain_rule}
On measure space $(\Omega,\AAcal,Q)$, let $g \geq 0$ be a Borel-measurable function, and define measure $P$ by
\begin{align*}
P(A) = \int_{A} g \, dQ, \qquad A \in \AAcal.
\end{align*}
For any Borel-measurable function $f$ on $\Omega$, it follows that
\begin{align*}
\int_{\Omega} f \, dP = \int_{\Omega} f g \, dQ.
\end{align*}
\end{lem}
\begin{proof}
See section 2.2, problem 4 of \citet{ash2000a}.
\end{proof}

\begin{lem}[Non-negativity of relative entropy]\label{lem:KL_nonneg}
For any probabilities $P$ and $Q$, we have $\KL(P;Q) \geq 0$.
\end{lem}
\begin{proof}[Proof of Lemma \ref{lem:KL_nonneg}]
If $P \ll Q$ does not hold, then $\KL(P;Q)=\infty$ and non-negativity follows trivially. As for the case of $P \ll Q$, we begin with the basic logarithmic inequality $x < (1+x)\log(1+x)$ for any $x > -1$ \citep{abramowitz1964a}. We thus have $x-1 < x\log(x)$ for any $x > 0$. Using this inequality and the chain rule (Lemma \ref{lem:chain_rule}), we have
\begin{align*}
\KL(P;Q) & = \exx_{P}\log\frac{dP}{dQ}\\
& = \exx_{Q}\frac{dP}{dQ}\log\frac{dP}{dQ}\\
& \geq \exx_{Q}\left( \frac{dP}{dQ} - 1 \right)\\
& = 0.
\end{align*}
The final equality uses the Radon-Nikodym theorem. 
\end{proof}

\begin{lem}[Lower bound on Bernoulli relative entropy]\label{lem:KL_lowbd_bernoulli}
The relative entropy between $\text{Bernoulli}(p)$ and $\text{Bernoulli}(q)$ is bounded below by $\KL(p;q) \geq 2(p-q)^{2}$.
\end{lem}
\begin{proof}[Proof of Lemma \ref{lem:KL_lowbd_bernoulli}]
Consider the function $f(p,q)$ defined
\begin{align*}
f(p,q) \defeq \KL(p;q) - 2(p-q)^{2}.
\end{align*}
Fix any arbitrary $p \in (0,1)$, and take the derivative with respect to $q$, noting that
\begin{align*}
\frac{d}{dq} f(p,q) = (-1)(p-q) \left(\frac{1}{q(1-q)}-4\right).
\end{align*}
Using the basic fact that $q(1-q) \leq 1/4$ for all $q \in (0,1)$, we have that the factor $(q(1-q))^{-1}-4$ is non-negative. Thus, the slope is negative when $p > q$, postive when $p < q$, and zero when $p = q$. Thus this is the only minimum of the function in $q$. Note that $f(p,p)=0$, and so for all $q \in (0,1)$ it follows that $f(p,q) \geq 0$. This holds for any choice of $p$ as well, implying the desired result by the definition of $f$.
\end{proof}

\begin{lem}[Chernoff bound for Bernoulli data]\label{lem:chernoff_tight}
Let $x_{1},\ldots,x_{n}$ be independent and identically distributed random variables, taking values $x \in \{0,1\}$. Write $\overbar{x} \defeq n^{-1} \sum_{i=1}^{n} x_{i}$ for the sample mean. The tails of the sample mean deviations can be bounded as
\begin{align*}
\prr\left\{ \overbar{x} - \exx x > \varepsilon \right\} & \leq \exp\left(-2n\varepsilon^{2}\right)\\
\prr\left\{ \overbar{x} - \exx x < -\varepsilon \right\} & \leq \exp\left(-2n\varepsilon^{2}\right)
\end{align*}
for any $0 < \varepsilon < 1-\exx x$.
\end{lem}
\begin{proof}[Proof of Lemma \ref{lem:chernoff_tight}]
For random variable $x \sim \text{Bernoulli}(\theta)$, recall that using Markov's inequality, for any $t > 0$ we have
\begin{align*}
\prr\{X > \varepsilon\} & = \prr\{\exp(tX) > \exp(t\varepsilon)\}\\
 & \leq \exp(-t\varepsilon) \exx e^{tX}\\
 & = \exp(-t\varepsilon)\left( 1-\theta + \theta e^{t} \right).
\end{align*}
Taking the derivative of this upper bound with respect to $t$ and setting it to zero, we obtain the condition
\begin{align*}
t^{\ast}(\varepsilon) = \log\left(\frac{\varepsilon}{\theta}\right)\left(\frac{1-\theta}{1-\varepsilon}\right),
\end{align*}
where we write $t^{\ast}(\varepsilon)$ to emphasize the dependence of $t^{\ast}$ on $\varepsilon$. We must have $t^{\ast}(\varepsilon) > 0$ for the bounds to hold. The value being passed into the $\log$ function must be greater than one. Fortunately, some simple re-arranging of factors shows that
\begin{align*}
\left(\frac{\varepsilon}{\theta}\right)\left(\frac{1-\theta}{1-\varepsilon}\right) > 1 \iff \varepsilon > \theta.
\end{align*}
So we have $t^{\ast}(\varepsilon) > 0$ whenever $\theta < \varepsilon < 1$. Plugging this in, some algebra shows that
\begin{align*}
\exp(-t^{\ast}\varepsilon)\left( 1-\theta + \theta e^{t^{\ast}} \right) & = \exp\left( (1-\varepsilon) \log\left(\frac{1-\theta}{1-\varepsilon} \right) + \varepsilon \log\left(\frac{\theta}{\varepsilon} \right) \right)\\
& = \exp\left(-\KL(\varepsilon;\theta)\right)
\end{align*}
where we note that the form given in precisely the relative entropy between $\text{Bernoulli}(\varepsilon)$ and $\text{Bernoulli}(\theta)$.

Returning to the setting of interest with $x_{1},\ldots,x_{n}$ and the sample mean $\overbar{x}$, note that using Markov's inequality again and the iid assumption on the data, we have
\begin{align*}
\prr\{ \overbar{x} > \theta + \varepsilon \} & = \prr\left\{ \sum_{i=1}^{n}x_{i} > n(\theta + \varepsilon) \right\}\\
& \leq \left( \exp(-t(\theta+\varepsilon)) \exx_{\ddist}e^{tx} \right)^{n}.
\end{align*}
Setting $t=t^{\ast}(\varepsilon+\theta)$ then, and using a classical lower bound on the relative entropy (Lemma \ref{lem:KL_lowbd_bernoulli}), we obtain
\begin{align}
\nonumber
\prr\{ \overbar{x} > \theta + \varepsilon \} & \leq \left( \exp\left(-\KL(\theta+\varepsilon;\theta)\right) \right)^{n}\\
\nonumber
& \leq \left( \exp\left(-2((\theta+\varepsilon)-\theta)^{2}\right) \right)^{n}\\
\label{eqn:chernoff_tight_upbd}
& = \exp\left(-2n\varepsilon^{2}\right).
\end{align}
Note that since $\varepsilon+\theta > \theta$ for all $\varepsilon > 0$, it follows that $t^{\ast}(\varepsilon+\theta) > 0$ for all $0 < \varepsilon < 1-\theta$.

Next we seek a lower bound on $\overbar{x}-\theta$, equivalently an upper bound on $-\overbar{x}+\theta$. This can be done by essentially the same process. Again for $X \sim \text{Bernoulli}(\theta)$, using Markov's inequality, we have for any $s > 0$ that
\begin{align*}
\prr\{X - \theta < -\varepsilon\} & = \prr\{-X > \varepsilon-\theta\}\\
& = \prr\{\exp(-sX) > \exp(s(\varepsilon-\theta))\}\\
 & \leq \exp(-s(\varepsilon-\theta)) \exx e^{-sX}\\
 & = \exp(s(\theta-\varepsilon))\left( 1-\theta + \theta e^{-s} \right).
\end{align*}
This is, of course, a rather familiar form. Writing $a = \theta-\varepsilon$, note that the function
\begin{align*}
\exp(sa)\left( 1-\theta + \theta e^{-s} \right)
\end{align*}
is minimized as a function of $s$ at
\begin{align*}
s^{\ast} = \log\left(\frac{1-a}{1-\theta}\right)\left(\frac{\theta}{a}\right),
\end{align*}
which analogous to earlier in the proof, satisfies $s^{\ast} > 0$ only when $\theta > a = \theta-\varepsilon$, which is to say whenever $\varepsilon > 0$. Keeping with the $a$ notation, note that plugging in $s^{\ast}$ to the bound above, we have
\begin{align*}
\exp(s^{\ast}a)\left( 1-\theta + \theta e^{-s^{\ast}} \right) & = \exp\left( (1-a) \log\left(\frac{1-\theta}{1-a} \right) + a \log\left(\frac{\theta}{a} \right) \right)\\
& = \exp\left(-\KL(a;\theta)\right),
\end{align*}
the exact same bound as before. It follows that
\begin{align*}
\prr\{ \overbar{x} - \theta < -\varepsilon \} & = \prr\{ -\overbar{x} > \varepsilon - \theta\}\\
& = \prr\left\{ -\sum_{i=1}^{n}x_{i} > n(\varepsilon-\theta) \right\}\\
& \leq \left( \exp(s(\theta-\varepsilon)) \exx_{\ddist}e^{-sx} \right)^{n}.
\end{align*}
Setting $s=t^{\ast}$ with $a = \theta-\varepsilon$, in a form analogous to the upper bounds done earlier, we have
\begin{align}
\nonumber
\prr\{ \overbar{x} - \theta < -\varepsilon \} & \leq \left( \exp\left(-\KL(\theta-\varepsilon;\theta)\right) \right)^{n}\\
\nonumber
& \leq \left( \exp\left(-2((\theta-\varepsilon)-\theta)^{2}\right) \right)^{n}\\
\label{eqn:chernoff_tight_lowbd}
& = \exp\left(-2n\varepsilon^{2}\right).
\end{align}
Taking a union bound over the two ``bad events'' in (\ref{eqn:chernoff_tight_upbd}) and (\ref{eqn:chernoff_tight_lowbd}), we have
\begin{align*}
\prr\{ |\overbar{x} - \theta| < -\varepsilon \} & \leq \prr\{\overbar{x} - \theta < -\varepsilon\} \cup \prr\{\overbar{x} - \theta > \varepsilon\}\\
& \leq \prr\{\overbar{x} - \theta < -\varepsilon\} + \prr\{\overbar{x} - \theta > \varepsilon\}\\
& \leq 2\exp\left(-2n\varepsilon^{2}\right),
\end{align*}
concluding the proof.
\end{proof}

\paragraph{Fundamental PAC-Bayes identity}

The following identity is fundamental to theoretical PAC-Bayesian analysis, and is a well-known result. \citet[p.~159--160]{catoni2004SLT} for example gives a concise proof, but for completeness, we provide a step-by-step proof of this result here. The key elements of the following theorem are the prior $\prior \in \proball$, and candidate posterior $\post \in \proball$. 

\begin{thm}\label{thm:main_PB_identity}
For any measurable function $h$,
\begin{align*}
\log \exx_{\prior} \exp(h) = \sup_{\post \in \proball} \left(\sup_{b \in \RR} \exx_{\post}(b \wedge h) - K(\post;\prior) \right).
\end{align*}
In the special case where $h$ is bounded above, then the above equality simplifies to
\begin{align*}
\log \exx_{\prior} \exp(h) = \sup_{\post \in \proball} \left( \exx_{\post}h - K(\post;\prior) \right).
\end{align*}
\end{thm}
\begin{proof}[Proof of Theorem \ref{thm:main_PB_identity}]
The key to this proof is a simple expansion of the relative entropy between an arbitrary $\post \in \proball$ and a specially modified prior $\prior^{\ast}$. This $\prior^{\ast}$ is defined in terms of the following requirement on the density function $d\prior^{\ast}/d\prior$: almost everywhere $[\prior]$, we must have
\begin{align*}
\left(\frac{d\prior^{\ast}}{d\prior}\right)(\omega) = g^{\ast}(\omega) \defeq \frac{\exp(h(\omega))}{\exx_{\prior}\exp(h)}.
\end{align*}
Satisfying this is easy by construction. Just define $\prior^{\ast}$ using $g^{\ast}$, as
\begin{align*}
\prior^{\ast}(A) \defeq \int_{A} g^{\ast} \, d\prior, \qquad A \in \AAcal.
\end{align*}
Since $g^{\ast} \geq 0$, it follows that $\prior^{\ast}$ is non-negative, and thus a measure on $(\Omega,\AAcal)$. As long as $\exp(h)$ is $\prior$-integrable, we have
\begin{align*}
\int_{A} g^{\ast} \, d\prior = \left(\exx_{\prior}\exp(h)\right)^{-1} \int_{A} \exp(h(\omega)) \, d\prior(\omega) \leq 1,
\end{align*}
and also that $\prior^{\ast}(\Omega) = 1$, so $\prior^{\ast} \in \proball$. Furthermore, note that $\prior^{\ast} \ll \prior$ and $\prior \ll \prior^{\ast}$.

Now, before proving all the necessary facts, let us run through the primary step of the argument using the following series of identities, which should be rather intuitive even at first glance:
\begin{align}
\nonumber
\KL(\post;\prior^{\ast}) & = \exx_{\post} \log\left(\frac{d\post}{d\prior^{\ast}}\right)\\
\label{eqn:main_PB_identity_use_chain}
& = \exx_{\post} \log\left(\frac{d\post}{d\prior}\frac{d\prior}{d\prior^{\ast}}\right)\\
%\label{eqn:main_PB_identity_simple_density}
& = \exx_{\post}\left( \log\frac{d\post}{d\prior} + \log\frac{d\prior}{d\prior^{\ast}} \right)\\
\nonumber
& = \exx_{\post}\left( \log\frac{d\post}{d\prior} + \log\exx_{\prior}\exp(h) - h \right).
\end{align}
When the left-hand side is finite, so is the right-hand side, and they are equal. Furthermore, when the left-hand side is infinite, so is the right-hand side.

To prove the above chain of equalities, first start by writing $g(\omega) = (d\prior^{\ast}/d\prior)(\omega)$, and observe that by the chain rule (Lemma \ref{lem:chain_rule}), we have
\begin{align*}
\int_{A} \left(\frac{1}{g(\omega)}\right) \, d\prior^{\ast} = \int_{A} \left(\frac{1}{g(\omega)}\right) g(\omega) \, d\prior(\omega) = \prior(A),
\end{align*}
for any $A \in \AAcal$. By $\prior \ll \prior^{\ast}$ and the Radon-Nikodym theorem, it follows that almost everywhere $[\prior]$, we have
\begin{align}\label{eqn:main_PB_identity_density_reciprocal}
\left(\frac{d\prior}{d\prior^{\ast}}\right)(\omega) = \frac{1}{g(\omega)} = \frac{\exx_{\prior}\exp(h)}{\exp(h(\omega))},
\end{align}
which justifies writing $d\prior/d\prior^{\ast} = 1/(d\prior^{\ast}/d\prior)$. Another basic fact using the chain rule (Lemma \ref{lem:chain_rule}) is that for each $A \in \AAcal$,
\begin{align*}
\int_{A} \left(\frac{d\post}{d\prior}\right)(\omega)\left(\frac{\exx_{\prior}\exp(h)}{\exp(h(\omega))}\right) \, d\prior^{\ast}(\omega) & = \int_{A} \left(\frac{d\post}{d\prior}\right)(\omega)\left(\frac{\exx_{\prior}\exp(h)}{\exp(h(\omega))}\right) \left(\frac{\exp(h(\omega))}{\exx_{\prior}\exp(h)}\right) \, d\prior(\omega)\\
& = \int_{A} \frac{d\post}{d\prior} \, d\prior\\
& = \post(A)\\
& = \int_{A} \frac{d\post}{d\prior^{\ast}} \, d\prior^{\ast}
\end{align*}
where the final three equalities follow from the Radon-Nikodym theorem and $\post \ll \prior$ and $\post \ll \prior^{\ast}$. Taking this basic fact and plugging in (\ref{eqn:main_PB_identity_density_reciprocal}), we have
\begin{align*}
\int_{A} \frac{d\post}{d\prior}\frac{d\prior}{d\prior^{\ast}} \, d\prior^{\ast} = \int_{A} \frac{d\post}{d\prior^{\ast}} \, d\prior^{\ast}, \qquad A \in \AAcal
\end{align*}
and then by uniqueness of the density function, that almost everywhere $[\prior^{\ast}]$,
\begin{align*}
\frac{d\post}{d\prior}\frac{d\prior}{d\prior^{\ast}} = \frac{d\post}{d\prior^{\ast}}.
\end{align*}
Since any statement a.e.~$[\prior^{\ast}]$ holds a.e.~$[\post]$ by $\post \ll \prior^{\ast}$, this proves (\ref{eqn:main_PB_identity_use_chain}).

The first equality holds from the definition of relative entropy, and with (\ref{eqn:main_PB_identity_use_chain}) now established, the remaining two equalities follow immediately from (\ref{eqn:main_PB_identity_density_reciprocal}).

The next step is to show that we can meaningfully write
\begin{align}\label{eqn:main_PB_identity_linearity}
\exx_{\post}\left( \log\frac{d\post}{d\prior} + \log\exx_{\prior}\exp(h) - h \right) = \KL(\post;\prior) + \log\exx_{\prior}\exp(h) - \exx_{\post} h
\end{align}
in the sense that both sides are well-defined, and take on equal values in $\RR\cup\{\infty\}$. To prove this, we would like to use the basic additivity property of Lebesgue integrals \citep[Theorem 1.6.3]{ash2000a}. First observe that the integrand of the left-hand side is well-defined and equal to $\KL(\post;\prior^{\ast})$. We need to show that the right-hand side is also well-defined. The first term $\KL(\post;\prior) \geq 0 > -\infty$ by Lemma \ref{lem:KL_nonneg}, and thus while it cannot be $-\infty$, it takes values in $\RR \cup \{\infty\}$. The remaining term depends on $h$. In the case that $h$ is bounded above, we have that $\exx_{\post} h < \infty$, meaning that the right-hand side of (\ref{eqn:main_PB_identity_linearity}) is well-defined, which implies via additivity that both sides of (\ref{eqn:main_PB_identity_linearity}) take values in $\RR\cup\{\infty\}$, and are equal in both the finite and infinite cases.

Note that when $h$ is not bounded above, this leaves the possibility that $\exx_{\post} h = \infty$, which would lead to the ambiguous $\infty - \infty$ on the right-hand side of (\ref{eqn:main_PB_identity_linearity}), spoiling the additivity property.

With the assumption of $h$ bounded above, and re-arranging some terms, we can write
\begin{align*}
\KL(\post;\prior^{\ast}) = \log\exx_{\prior}\exp(h) - \left(\exx_{\post}h - \KL(\post;\prior)\right).
\end{align*}
By non-negativity of the relative entropy (Lemma \ref{lem:KL_nonneg}), the left-hand side is minimized when $\post = \prior^{\ast}$, in which case it takes the value $\KL(\prior^{\ast};\prior^{\ast}) = 0$. Note that as $\prior^{\ast} \in \proball$, the supremum of the term in parentheses on the right-hand side is achieved at $\post = \prior^{\ast}$. This means we can write
\begin{align}
\label{eqn:main_PB_identity_bounded_case}
\log\exx_{\prior}\exp(h) = \sup_{\post \in \proball}\left( \exx_{\post}h - \KL(\post;\prior) \right)
\end{align}
for $h$ bounded above.

To complete the proof, we must consider the case where $h$ is unbounded. As preparation, create a measurable function sequence $(h_{k})$ defined by $h_{k} = b_{k} \wedge h$, where $(b_{k})$ satisfies $b_{k} \uparrow \infty$ and is increasing. Since we have
\begin{align*}
\lim\limits_{k \to \infty} \exp(h_{k}(\omega)) = \exp(h(\omega))
\end{align*}
pointwise in $\omega \in \Omega$, and $h_{k} \leq h_{k+1} \leq \ldots \leq h$ for any $k$, by the monotone convergence theorem, we have
\begin{align*}
\lim\limits_{k \to \infty} \exx_{\prior}\exp(h_{k}) = \exx_{\prior}\exp(h),
\end{align*}
and using the continuity of the log function,
\begin{align*}
\lim\limits_{k \to \infty} \log\exx_{\prior} \exp(h_{k}) = \log\left( \lim\limits_{k \to \infty}\exx_{\prior}\exp(h_{k}) \right) = \log\exx_{\prior}\exp(h).
\end{align*}
This means we can write
\begin{align}
\nonumber
\log\exx_{\prior} \exp(h) & = \sup_{b \in \RR} \log\exx_{\prior} \exp(b \wedge h)\\
\label{eqn:main_PB_identity_unbd_case}
& = \sup_{b \in \RR} \sup_{\post \in \proball} \left( \exx_{\post}(b \wedge h) - \KL(\post;\prior) \right)\\
\label{eqn:main_PB_identity_sup_swap}
& = \sup_{\post \in \proball} \sup_{b \in \RR} \left( \exx_{\post}(b \wedge h) - \KL(\post;\prior) \right)\\
\nonumber
& = \sup_{\post \in \proball} \left( \sup_{b \in \RR}\exx_{\post}(b \wedge h) - \KL(\post;\prior) \right).
\end{align}
Since for any $b \in \RR$, we have that $b \wedge h \leq b < \infty$, we can use (\ref{eqn:main_PB_identity_bounded_case}), the key identity for the case of bounded functions, which immediately implies (\ref{eqn:main_PB_identity_unbd_case}). Finally, regarding the swap of supremum operations, note that the function of interest is
\begin{align*}
f(\post,b) = \exx_{\post}(b \wedge h) - \KL(\post;\prior), \qquad (\post,b) \in \proball \times \RR.
\end{align*}
For an arbitrary sequence $(\post_{k},b_{k})$, observe that for all $k$,
\begin{align*}
f(\post_{k},b_{k}) & \leq \sup_{\post} f(\post,b_{k}) \leq \sup_{b} \sup_{\post} f(\post,b)\\
f(\post_{k},b_{k}) & \leq \sup_{b} f(\post_{k},b) \leq \sup_{\post} \sup_{b} f(\post,b).
\end{align*}
If $f$ is unbounded on $\proball \times \RR$, then the sequence $(\post_{k},b_{k})$ can be constructed such that $f(\post_{k},b_{k}) \to \infty$ as $k \to \infty$, implying that in both cases the supremum is infinite, so equality holds trivially. On the other hand, when $f$ is bounded above, the sequence can be constructed such that $f(\post_{k},b_{k}) \to B$, and so the above inequalities imply
\begin{align*}
B & = \lim\limits_{k \to \infty} f(\post_{k},b_{k}) \leq \sup_{b} \sup_{\post} f(\post,b) \leq B\\
B & = \lim\limits_{k \to \infty} f(\post_{k},b_{k}) \leq \sup_{\post} \sup_{b} f(\post,b) \leq B
\end{align*}
and thus, as desired, the step to (\ref{eqn:main_PB_identity_sup_swap}) holds. This concludes the chain of equalities and the proof.
\end{proof}

\paragraph{A relative entropy bound}

One key technique based on the identity in Theorem \ref{thm:main_PB_identity} is given as follows, closely following the approach of \citet{catoni2017a}. Assume we are given some independent data $x_{1},\ldots,x_{n}$, assumed to be copies of the random variable $x \sim \ddist$. In addition, let $\epsilon_{1},\ldots,\epsilon_{n}$ similarly be independent observations of ``strategic noise'' of sorts, with distribution $\epsilon \sim \post$ that we can design. 

\begin{proof}[Proof of Lemma \ref{lem:cat17type_KLbound}]
Start with the following elementary inequality: if $X$ is a random variable such that $\exx e^{X} \leq 1$, then for any $\delta \in (0,1)$, we have that $X$ exceeds $\log(\delta^{-1})$ with probability no greater than $\delta$. To see this, observe that
\begin{align}
\label{eqn:exp1_tails}
\prr\{X \geq \log(\delta^{-1}) \} = \prr\{\exp(X) \geq 1/\delta\} = \exx I\{\delta\exp(X) \geq 1\} \leq \exx \delta e^{X}  \leq \delta.
\end{align}
Next, we set the function $h$ in Theorem \ref{thm:main_PB_identity} to be a sum of functions depending on both the data and the noise, as
\begin{align*}
h(\epsilon) = \sum_{i=1}^{n}f(x_{i},\epsilon) - n\log\exx_{\ddist}\exp(f(x,\epsilon)).
\end{align*}
Since $f$ is bounded on $\RR^{2}$ by hypothesis, we have that $h$ is also bounded. Using Theorem \ref{thm:main_PB_identity}, we have
\begin{align*}
B_{0} & \defeq \sup_{\post \in \proball} \left( \exx_{\rho}h(\epsilon) - K(\post;\prior) \right)\\
& = \log \exx_{\prior} \left( \frac{\exp\left(\sum_{i=1}^{n}f(x_{i},\epsilon) \right)}{(\exx_{\ddist}\exp(f(x,\epsilon)))^{n}} \right).
\end{align*}
Next, taking expectation with respect to the sample, observe that
\begin{align*}
\exx\exp(B_{0}) & = \exx \int \left( \frac{\exp\left(\sum_{i=1}^{n}f(x_{i},\epsilon) \right)}{(\exx_{\ddist}\exp(f(x,\epsilon)))^{n}} \right) \, \prior(\epsilon)\\
& = \int \left( \frac{\exx\exp\left(\sum_{i=1}^{n}f(x_{i},\epsilon) \right)}{(\exx_{\ddist}\exp(f(x,\epsilon)))^{n}} \right) \, \prior(\epsilon)\\
& = 1.
\end{align*}
The above equalities follow from straightforward algebraic manipulations, independence of the data, and taking the integration over the sample inside the integration over the noise, valid using Fubini's theorem. Applying (\ref{eqn:exp1_tails}) with $X=B_{0}$, noting that the only randomness is due to the sample, it holds that for probability at least $1-\delta$, uniform in the choice of $\post$, we have%
\begin{align*}
\exx_{\rho}h(\epsilon) - K(\post;\prior) \leq \log(\delta^{-1}).
\end{align*}
Plugging in the above definition of $h$ and dividing by $n$, we have
\begin{align*}
\frac{1}{n} \sum_{i=1}^{n} \int f(x_{i},\epsilon) \, d\post(\epsilon) \leq \int \log\exx_{\ddist}\exp(f(x,\epsilon)) \, d\post(\epsilon) + \frac{\KL(\post;\prior) + \log(\delta^{-1})}{n}.
\end{align*}
Finally, since the noise observations are iid, we have
\begin{align*}
\frac{1}{n} \sum_{i=1}^{n} \int f(x_{i},\epsilon) \, d\post(\epsilon) & = \frac{1}{n} \sum_{i=1}^{n} \int f(x_{i},\epsilon_{i}) \, d\post(\epsilon_{i})\\
& = \exx \left( \frac{1}{n} \sum_{i=1}^{n} f(x_{i},\epsilon_{i}) \right)
\end{align*}
with expectation over the noise sample. This equality yields the desired result.
\end{proof}

\bibliographystyle{apalike}
\bibliography{refs.bib}

\begin{thebibliography}{}

\bibitem[Abramowitz and Stegun, 1964]{abramowitz1964a}
Abramowitz, M. and Stegun, I.~A. (1964).
\newblock {\em Handbook of Mathematical Functions With Formulas, Graphs, and
  Mathematical Tables}, volume~55 of {\em National Bureau of Standards Applied
  Mathematics Series}.
\newblock US National Bureau of Standards.

\bibitem[Alquier and Guedj, 2018]{alquier2018a}
Alquier, P. and Guedj, B. (2018).
\newblock Simpler {PAC}-{B}ayesian bounds for hostile data.
\newblock {\em Machine Learning}, 107(5):887--902.

\bibitem[Alquier et~al., 2016]{alquier2016a}
Alquier, P., Ridgway, J., and Chopin, N. (2016).
\newblock On the properties of variational approximations of {G}ibbs
  posteriors.
\newblock {\em Journal of Machine Learning Research}, 17(1):8374--8414.

\bibitem[Ash and Doleans-Dade, 2000]{ash2000a}
Ash, R.~B. and Doleans-Dade, C. (2000).
\newblock {\em Probability and Measure Theory}.
\newblock Academic Press.

\bibitem[B{\'e}gin et~al., 2016]{begin2016a}
B{\'e}gin, L., Germain, P., Laviolette, F., and Roy, J.-F. (2016).
\newblock {PAC}-{B}ayesian bounds based on the {R}{\'e}nyi divergence.
\newblock In {\em Proceedings of Machine Learning Research}, volume~51, pages
  435--444.

\bibitem[Boucheron et~al., 2013]{boucheron2013a}
Boucheron, S., Lugosi, G., and Massart, P. (2013).
\newblock {\em Concentration inequalities: a nonasymptotic theory of
  independence}.
\newblock Oxford University Press.

\bibitem[Catoni, 2004]{catoni2004SLT}
Catoni, O. (2004).
\newblock {\em Statistical learning theory and stochastic optimization: {E}cole
  d'Et{\'e} de Probabilit{\'e}s de {S}aint-{F}lour {XXXI}-2001}, volume 1851 of
  {\em Lecture Notes in Mathematics}.
\newblock Springer.

\bibitem[Catoni, 2007]{catoni2007a}
Catoni, O. (2007).
\newblock {\em Pac-{B}ayesian Supervised Classification: The Thermodynamics of
  Statistical Learning}.
\newblock IMS Lecture Notes--Monograph Series. Institute of Mathematical
  Statistics.

\bibitem[Catoni, 2012]{catoni2012a}
Catoni, O. (2012).
\newblock Challenging the empirical mean and empirical variance: a deviation
  study.
\newblock {\em Annales de l'Institut Henri Poincar{\'e}, Probabilit{\'e}s et
  Statistiques}, 48(4):1148--1185.

\bibitem[Catoni and Giulini, 2017]{catoni2017a}
Catoni, O. and Giulini, I. (2017).
\newblock Dimension-free {PAC}-{B}ayesian bounds for matrices, vectors, and
  linear least squares regression.
\newblock {\em arXiv preprint arXiv:1712.02747}.

\bibitem[Devroye et~al., 2016]{devroye2016a}
Devroye, L., Lerasle, M., Lugosi, G., and Oliveira, R.~I. (2016).
\newblock Sub-gaussian mean estimators.
\newblock {\em Annals of Statistics}, 44(6):2695--2725.

\bibitem[Germain et~al., 2016]{germain2016a}
Germain, P., Bach, F., Lacoste, A., and Lacoste-Julien, S. (2016).
\newblock {PAC}-{B}ayesian theory meets {B}ayesian {I}nference.
\newblock In {\em Advances in Neural Information Processing Systems 29}, pages
  1884--1892.

\bibitem[Maurer, 2004]{maurer2004a}
Maurer, A. (2004).
\newblock A note on the {PAC} {B}ayesian theorem.
\newblock {\em arXiv preprint arXiv:cs/0411099}.

\bibitem[McAllester, 1999]{mcallester1999a}
McAllester, D.~A. (1999).
\newblock Some {PAC}-{B}ayesian theorems.
\newblock {\em Machine Learning}, 37(3):355--363.

\bibitem[McAllester, 2003]{mcallester2003a}
McAllester, D.~A. (2003).
\newblock {PAC}-{B}ayesian stochastic model selection.
\newblock {\em Machine Learning}, 51(1):5--21.

\bibitem[McAllester, 2013]{mcallester2013a}
McAllester, D.~A. (2013).
\newblock A {PAC}-{B}ayesian tutorial with a dropout bound.
\newblock {\em arXiv preprint arXiv:1307.2118}.

\bibitem[Seeger, 2002]{seeger2002a}
Seeger, M. (2002).
\newblock {PAC}-{B}ayesian generalisation error bounds for {G}aussian process
  classification.
\newblock {\em Journal of Machine Learning Research}, 3(Oct):233--269.

\bibitem[Seldin et~al., 2012]{seldin2012a}
Seldin, Y., Laviolette, F., Cesa-Bianchi, N., Shawe-Taylor, J., and Auer, P.
  (2012).
\newblock {PAC}-{B}ayesian inequalities for martingales.
\newblock {\em IEEE Transactions on Information Theory}, 58(12):7086--7093.

\bibitem[Shawe-Taylor et~al., 1996]{shawetaylor1996a}
Shawe-Taylor, J., Bartlett, P.~L., Williamson, R.~C., and Anthony, M. (1996).
\newblock A framework for structural risk minimisation.
\newblock In {\em Proceedings of the 9th Annual Conference on Computational
  Learning Theory}, pages 68--76. ACM.

\bibitem[Tolstikhin and Seldin, 2013]{tolstikhin2013a}
Tolstikhin, I. and Seldin, Y. (2013).
\newblock {PAC}-{B}ayes-{E}mpirical-{B}ernstein inequality.
\newblock In {\em Advances in Neural Information Processing Systems 26}, pages
  109--117.

\bibitem[Valiant, 1984]{valiant1984a}
Valiant, L.~G. (1984).
\newblock A theory of the learnable.
\newblock {\em Communications of the ACM}, 27(11):1134--1142.

\end{thebibliography}

\end{document}